\documentclass{article}

\usepackage[main,final]{neurips_2026}

\usepackage[utf8]{inputenc} % allow utf-8 input
\usepackage[T1]{fontenc}    % use 8-bit T1 fonts
\usepackage{hyperref}       % hyperlinks
\usepackage{url}            % simple URL typesetting
\usepackage{booktabs}       % professional-quality tables
\usepackage{nicefrac}       % compact symbols for 1/2, etc.
\usepackage{microtype}      % microtypography
\usepackage{xcolor} 

\usepackage{graphicx}

\usepackage{amsmath,amssymb,amsfonts,amsthm,mathtools,bm}
\usepackage{bbm}
\usepackage{braket}
\usepackage{tcolorbox}

\usepackage[ruled,vlined]{algorithm2e}
\SetKw{Return}{return}

\theoremstyle{plain}
\newtheorem{theorem}{Theorem}
\newtheorem{lemma}{Lemma}
\newtheorem{corollary}{Corollary}

\theoremstyle{definition}
\newtheorem{definition}{Definition}

\DeclareMathOperator*{\argmin}{argmin}

\usepackage{hyperref}

\newcommand{\pc}{\mathcal{P}}

\newcommand{\rc}{\mathcal{R}}

\newcommand{\lc}{\mathcal{L}}

\newcommand{\utheta}{\underline{\theta}}
\newcommand{\usigma}{\underline{\sigma}}
\newcommand{\hutheta}{\underline{\hat{\theta}}}

\newcommand{\EX}{ \mathbb{E}}

\newcommand{\Emu}{\underset{\usigma \sim \mu}{ \mathbb{E}} }
\newcommand{\EXp}[1]{\underset{#1}{ \mathbb{E}} }
\newcommand{\Var}[1]{\underset{#1}{\operatorname{Var}} }

\newcommand{\uDelta}{\underline{\Delta}}

\long\def\ca#1\cb{}% colors

\title{Computationally sufficient statistics for Ising models}

\author{Abhijith Jayakumar\thanks{Email: abhijithj@lanl.gov}, ~ Shreya Shukla, ~Marc Vuffray, ~Andrey Y. Lokhov,~ Sidhant Misra \\ \\Theoretical Division, Los Alamos National Laboratory, USA}

\begin{document}

\maketitle

\begin{abstract}
  Learning Gibbs distributions using only sufficient statistics has long been recognized as a computationally hard problem. On the other hand, computationally efficient algorithms for learning Gibbs distributions rely on  access to full sample configurations generated from the model. For many systems of interest that arise in physical contexts, expecting a full sample to be observed is not practical, and hence it is important to look for computationally efficient methods that solve the learning problem with access to only a limited set of statistics. We examine the trade-offs between the power of computation and observation within this scenario, employing the Ising model as a paradigmatic example. We show that the couplings of the model can be learned to a desired accuracy from statistics whose order grows linearly with the $\ell_1$ width of the model and only logarithmically with the inverse of the accuracy. In particular, for a fixed accuracy, the required order is proportional to the width. This allows us to infer the model's structure, using statistics whose order grows logarithmically with the inverse of the weakest coupling, and also to learn its magnetic fields. When the width and the accuracy are fixed, the running time is polynomial in the number of variables. We also discuss a setting where prior information about structure of the model is available and show that the learning problem can be solved efficiently with even more limited observational power. 
\end{abstract}

\section{Introduction}
Trade-offs between sample and computational complexities of learning tasks has garnered much interest in learning theory (\cite{wein2025computational, agarwal2012computational}). However, in certain cases, the learner's access to the data is itself constrained by limited observational power, i.e., instead of observing full samples, they may only observe some aggregated statistics or low-order moments. This raises a fundamental question: how does this observational power interact with computational tractability?

We study this question in the context of learning Gibbs distributions over discrete variables, where such a setting naturally arises as an interpolation between two sets of prior works. The first are hardness results that show that learning the parameters of such a distribution given only sufficient statistics is a computationally intractable problem \cite{montanari_computational_2015, bresler_hardness_2014}. On the other hand, starting from the breakthrough by Bresler (\cite{bresler2015efficiently}), provably efficient algorithms that use full independent samples from the model are known for this problem. Several other works have also looked at this learning task from other types of data models like samples drawn from a dynamic (\cite{gaitonde2023unified}) or from a metastable distribution (\cite{jayakumar_discrete_2025}). All the efficient algorithms here assume that full sample configurations of the model can be observed. This motivates us to look for an intermediate regime that has been left unexplored by the aforementioned literature. We ask the question: 

\emph{Is it possible to  learn a discrete graphical model from a limited set of statistics, such that the computational effort scales as a polynomial in the number of variables and in the data size?}

We initiate the study of this question in the Ising class of models.
Here the distribution is over $p$  binary spins, $\usigma \in \{-1,1\}^p$. The distribution that we are attempting to learn is a Gibbs distribution associated with a general quadratic function over these variables,
\begin{align}
    \mu(\usigma) \propto \exp(E(\usigma; \utheta^*)) \ \ \text{where} \ \ E(\usigma; \utheta^*) \coloneqq \sum_{1\leq u <  v \leq p} \theta^*_{uv} \sigma_u \sigma _v + \sum_{u =1}^p \theta^*_{u} \sigma_u, \label{eq:basic1}
\end{align}
with the convention that $\theta^*_{uv} = \theta^*_{vu},$ that will be used throughout this paper.
Following previous works, our aim here will be to learn the parameters $\utheta^*$. We will assume that we know an upper bound on the \emph{$\ell_1$-width} of the model \cite{Klivans2017, goel2019learning}. That is, we know a $\gamma$ such that, $\sum_{v \neq u}  |\theta^*_{uv}| + |\theta^*_u| \leq \gamma, ~~\forall u \in [p] .$ 

We refer to the expectation values of monomials ($\EX \sigma_i$, $\EX \sigma_i \sigma_j, ~ \EX \sigma_i\sigma_j\sigma_k, \ldots$) as to the \emph{statistics} of this model. It is well known that the parameters of the model are entirely fixed by the first ($\EX \sigma_i$) and second order ($\EX \sigma_i \sigma_j$) statistics. However, performing the actual inversion from these statistics to the parameter has been generally proven to be a hard problem (\cite{montanari_computational_2015, bresler_hardness_2014}). Efficient estimators which are based on observing full sample configurations, on the other hand, can be implicitly viewed as having complete access to all the statistics of the model (up to order $p$). In this paper we will consider a setting where the learner is given access to only monomials up to some degree $d$.

For solving this learning problem, we will use a modification of the \emph{Interaction Screening Estimator} (ISE). The ISE, first introduced in \cite{vuffray2016interaction}, is an M-estimator that works by minimizing a convex loss function for every variable $u \in [p]$ subject to constraints on the $\ell_1$-width of the model. Given $n$ samples from the true distribution,
\begin{equation}
\lc_{u,n}(\utheta) \equiv \frac1n \sum_{t =1}^n e^{-\sigma^{(t)}_u(  \theta_u + \sum_{v\neq u} \theta_{u,v} \sigma^{(t)}_v  )},
\end{equation}
\begin{equation}
   \hutheta_u = \argmin_{\utheta: |\utheta| \leq \gamma} \lc_{u,n}(\theta).
\end{equation}
This constrained optimization problem is solved for each of the $p$ variables in the model to estimate the parameters in the energy function connected to it. It has been shown to efficiently recover the parameters and structure of Ising models and other discrete graphical models \cite{vuffray2016interaction, vuffray2019efficient}. As shown in these works, this estimator requires exact configurations to be evaluated. However, the exponential function, combined with the $\ell_1$ constraint, suggests the possibility of approximating the exponential function with a low-degree polynomial.

This idea is explored briefly in \cite{lokhov2018optimal} by expanding the interaction screening loss to second-order. At sufficiently low $\gamma$ this second-order approximation can be expected to approximate the full IS estimator well. But this strategy is bound to fail at higher values of $\gamma$. Moreover, expanding the exponential as a polynomial can potentially lead to a loss of convexity, and in turn to the loss of learning guarantees the output of ISE. If we could find a provably effective way to approximate the exponential in the interaction screening method using, say, a degree $d$ polynomial, then that implies that the parameters of the model can be estimated using statistics up to order $O(d)$. 

In this work, we propose to approximate the gradient of the IS loss using a polynomial approximation. This allows us to analyze this problem as a convex optimization problem with a corrupted gradient oracle. From the known robustness properties of gradient descent, we show that we can get close the true optimum using this approximate gradient. Our main result (Theorem~\ref{thm:main_thm}) shows that statistics up to order $O(\gamma + \log(1/\epsilon))$ suffice to learn every coupling to accuracy $\epsilon$. For a fixed accuracy, the required order is therefore proportional to the width $\gamma$. Exact recovery of the structure requires accuracy $\alpha/2$, where $\alpha$ is the smallest magnitude of a nonzero coupling, so the structure can be recovered from statistics of order $O(\gamma + \log(1/\alpha))$. The order needed to learn the magnetic fields similarly grows logarithmically with the inverse of the target accuracy (Theorem~\ref{thm:mag_learn}). When all the statistics are estimated from the same $n$ samples, $n = O\big(e^{8\gamma}\,\mathrm{poly}(\gamma, \log\tfrac{1}{\epsilon})\log(p/\delta)/\epsilon^4\big)$ samples suffice (Corollary~\ref{cor:sample_complexity}). This has the same scaling with $p$, $\gamma$ and $\epsilon$, up to factors polynomial in $\gamma$ and $\log(1/\epsilon)$, as Interaction Screening with full samples \cite{vuffray2019efficient}. The cost of acquiring the statistics in other ways, for instance from separate samples for each statistic, is discussed in Section~\ref{sec:moment_table}. The running time of the algorithm is polynomial in the size of the table of statistics, $p^{O(\gamma + \log(1/\epsilon))}$. This is polynomial in $p$ when $\gamma$ and $\epsilon$ are constants independent of $p$, which is the standard regime in the literature on learning Ising models \cite{Klivans2017, vuffray2016interaction}. 

Our results give rise to a natural question about corresponding lower bounds, i.e., what is the lowest order of statistics that will allow us to tractably learn the model. Our results point to the existence of a type of information-computation tradeoff in this learning problem. In particular, the power of statistics between information-theoretically sufficient (order $2$) and that is proven in this work to be computationally sufficient (order $O(\gamma)$ for a fixed accuracy) is currently not quantified in literature.

\subsection{Further Related Work}
\paragraph{Learning from statistical queries.}
The Statistical Query (SQ) model, introduced by Kearns (\cite{Kearns1998}), formalizes learning scenarios in which the algorithm accesses data only through expectations of bounded functions rather than individual sample. Follow-up works have shown that many learning problems that are information-theoretically solvable become computationally intractable when restricted to SQ access alone~(\cite{Feldman2013,Feldman2017}). These results show that the form in which data is observed can play a crucial role in determining computational feasibility.
Our setting differs from the classical SQ model in two aspects. First, we consider access to a fixed collection of low-degree moments rather than arbitrary adaptively chosen queries. Second, we allow these moments to be reused across multiple optimization steps. Nonetheless, both settings emphasize how restrictions on observable statistics affect the computational complexity of learning.
\paragraph{Mean-field heuristics.} An approach to the inverse Ising problem is to replace the true Gibbs measure with a tractable mean-field approximation, yielding explicit estimators for the parameters of the model purely  from empirical statistics (\cite{nguyen2017inverse}). In the simplest mean field reconstruction, inverting the correlation matrix gives a closed form estimate for the parameters of the models.
More accurate variants incorporate linear-response corrections and, in particular, the Onsager reaction term (\cite{ricci2012bethe, tanaka1998mean}). These methods are attractive for their simplicity, though they are ultimately uncontrolled approximations which fail at larger $\gamma$. 
\paragraph{Continuous variables and quantum generalizations.}
It is natural to consider the same problem in other types of alphabets. For continuous variable models, learning Gibbs distributions can be solved computationally efficiently using score matching, which can be viewed as morally equivalent to Interaction Screening or Pseudo-likelihood methods as it breaks down the global learning problem into single variable problems \cite{pabbaraju2023provable, hyvarinen2005estimation}. However, score matching does not require access to full samples to learn. For a distribution specified by a degree-$d$ polynomial energy function, score matching loss only requires access to statistics up to order $2d$. This shows an important distinction between continuous and discrete alphabet learning scenarios.

The quantum generalization of the Gibbs learning task has been shown to be tractable in recent works \cite{bakshi2024learning, narayanan2024improved}. These algorithms work by using a sum-of-squares hierarchy to match certain statistics of the true quantum gibbs state with those of a parametric hypothesis. This implicitly suggests, even in the classical setting, that the learning problem is solvable using a set of statistics strictly smaller than the complete set. However, in the quantum setting the structure of the model is assumed to be known and sparse. In the classical setting, this makes the learning problem manifestly tractable using some lower order moments that depend only on the sparsity. This setting is discussed in Section \ref{sec:physics_models}.  These considerations show that the minimal set of statistics required to solve such a problem will also depend strongly on the prior information we have about the true model.

\paragraph{Notation:}
We use underlines (e.g., $\usigma,\utheta$) to denote vectors. For $p\in\mathbb{N}$ we write $[p]=\{1,\dots,p\}$. The Ising configuration is $\usigma\in\{-1,1\}^p$ with coordinates $\sigma_u$ for $u\in[p]$. Parameters are $\utheta^*=(\{\theta^*_{u,v}\}_{u\neq v},\{\theta^*_u\}_u)$, with local parameter vector at node $u$ denoted $\utheta_u=(\theta_u,\{\theta_{u,v}\}_{v\neq u})$. Expectations under the true model are denoted $\EX[\cdot]$; empirical average from $n$ i.i.d. samples are $\EX_n[\cdot]$.

\section{Interaction screening with approximate gradients}

This is section we will give the main algorithm and its performance guarantees. We split the learning task into three parts:
\begin{enumerate}
    \item  Parameter learning of the second order terms ($\theta^*_{u,v})$ (Theorem \ref{thm:main_thm})
    \item Structure learning, i.e., finding the graph structure encoded in the $\theta_{u,v}^*$ terms (Theorem \ref{thm:struct_learn}) 
    \item Parameter learning of linear terms ($\theta_u^*$ ) (Theorem \ref{thm:mag_learn}). 
\end{enumerate}
We show that these tasks can be solved sequentially given access to data in terms of empirical estimates of the statistics of the model up to $O(\gamma)$. %Theorems that quantify the error propagation between these tasks are derived in this section.

We denote by $E_u(\usigma, \utheta_u) \coloneqq \sigma_u (\theta_u + \sum_{j \neq u} \theta_{uj} \sigma_j)$, a parametric ``local'' energy function. This, in effect, determines the parametrization of the single variable conditionals of the model. 
%$$ \lc_u(\utheta_u) = \EXp{\usigma \ \sim \ \mu} e^{- E_u(\usigma, \utheta_u)} $$
\subsection{Moment table and observation models}
\label{sec:moment_table}

The learner never sees the model $\mu$ directly. Its only input is a table of (approximate) low-order moments. For a moment order $D \in \mathbb{N}$, define the index set $\mathcal{K}_D \coloneqq \{K \subseteq [p] : |K| \le D\}$. For each index $K \in \mathcal{K}_D$, let $m_K \coloneqq \EX_{\mu}[\prod_{i \in K}\sigma_i]$ denote the true moment (with $m_\emptyset = 1$). Note that the size of the index set is, $|\mathcal{K}_D| = \sum_{j \le D}\binom{p}{j} \le (ep)^D$ for $D\ge 1$.

\begin{definition}[Approximate moment table]
\label{def:moment_table}
A random table $S = \{S[K]\}_{K \in \mathcal{K}_D}$ is a \emph{$(D,\tau,\delta)$-approximate moment table} of $\mu$ if
\[
\mathbb{P}\Big[\max_{K \in \mathcal{K}_D}\big|S[K] - m_K\big| \le \tau\Big] \;\ge\; 1-\delta .
\]
The probability is over the randomness of the data-collection procedure that produces $S$.
\end{definition}

The guarantee is \emph{simultaneous} over all entries of the table, and no independence between different entries is assumed. The algorithms in this paper take such a table as their only input; once the event in Definition~\ref{def:moment_table} holds, the rest of the analysis is deterministic. The accuracy and confidence parameters $(\tau,\delta)$ therefore encode everything the learner needs to know about how the table was obtained. Two natural ways of producing a table, which are the two extremes of a larger family, are:
\begingroup\renewcommand{\theenumi}{\alph{enumi}}\renewcommand{\labelenumi}{(\theenumi)}
\begin{enumerate}
    \item \textbf{Shared dataset.} Draw $n$ i.i.d.\ samples $\usigma^{(1)},\dots,\usigma^{(n)}\sim\mu$ and set $S_{D,n}[K] \coloneqq \EX_n[\prod_{i\in K}\sigma_i] = \frac1n\sum_{t=1}^n\prod_{i\in K}\sigma^{(t)}_i$ for every $K\in\mathcal{K}_D$.
    \item \textbf{Separate datasets.} For every $K\in\mathcal{K}_D$ draw a fresh set of $n$ i.i.d.\ samples and estimate only $m_K$ from it.
\end{enumerate}\endgroup
Hybrids (different groups of moments estimated from different, possibly overlapping, datasets, or with $K$-dependent sample sizes) are covered in the same way. We write $\EX_n$ for an empirical average over $n$ samples, and $S_{D,n}$ for a table whose entries are $n$-sample empirical moments, in either of the two senses above.

\begin{lemma}[Empirical tables are approximate moment tables]
\label{lem:emp_table}
Let $S_{D,n}$ be built as in (a) or in (b). If
\[
n \;\ge\; \frac{2}{\tau^2}\log\!\Big(\frac{2|\mathcal{K}_D|}{\delta}\Big),
\]
which holds in particular if $n \ge \frac{2}{\tau^2}\big(D\log(ep)+\log\frac2\delta\big)$, then $S_{D,n}$ is a $(D,\tau,\delta)$-approximate moment table.
\end{lemma}
\begin{proof}
Fix $K\in\mathcal{K}_D$. Each entry is an average of $n$ i.i.d.\ random variables $\prod_{i\in K}\sigma_i^{(t)}\in\{-1,1\}$ with mean $m_K$, so Hoeffding's inequality gives $\mathbb{P}[|S_{D,n}[K]-m_K|>\tau]\le 2e^{-n\tau^2/2}$. A union bound over the $|\mathcal{K}_D|$ entries (no independence between entries is needed) bounds the failure probability by $2|\mathcal{K}_D|e^{-n\tau^2/2}\le\delta$.
\end{proof}

\paragraph{Cost of acquiring the table.} The number of samples \emph{per moment} is the same, $n$, in (a) and (b), but the total measurement cost is not. In (a) each of the $n$ samples must be revealed in enough detail to evaluate every monomial of degree $\le D$, which in general means observing the sample in full. In (b) each observation is only needed on the $|K|\le D$ spins of one monomial, but the total number of observations is $n\,|\mathcal{K}_D| = n\,p^{O(D)}$. Our results are stated for a generic approximate moment table so that the cost of acquiring it can be accounted for separately in each such model.

To estimate  the $\theta_{u,v}^*$ parameters we use the projected gradient descent method with approximate gradients as outlined in Algorithm \ref{alg:ais-moments}.  The performance guarantees on the output of this algorithm is stated and proven in the Theorem below.

\begin{algorithm} [t]
% 1. Draw Top Line
\hrule height 1pt
\vspace{1mm}

% 2. Caption goes here
\caption{Interaction Screening with Moment-Based Approximate Gradients}
\label{alg:ais-moments}

% 3. Draw Separator Line
\vspace{1mm}
\hrule height 0.4pt
\vspace{2mm}

% Your existing algorithm2e setup
\SetAlgoLined
\DontPrintSemicolon

% Inputs
\textbf{Input:} Degree $d$, approximate moment table $S$ of order $D = d+2-(d\bmod 2)$ (Definition~\ref{def:moment_table}),
$\ell_1$ bound $\gamma$, Step size $\eta$, Iterations $T$\;

\textbf{Output:} Estimated parameters $\hutheta_u$ for every $u \in [p]$

\BlankLine

\For{$u = 1$ \textbf{to} $p$}{
    Initialize $\utheta^{(1)}_u = \{ \theta_u^{(1)}, \theta_{u,1}^{(1)}, \ldots, \theta_{u,p}^{(1)} \} \leftarrow 0$\;
    
    \For{$t = 1$ \textbf{to} $T$}{
        \tcp{Gradient Approximation}
        Construct $\tilde{\nabla}\lc_{u,S}(\utheta^{(t)}_u)$ from $S$ as follows:\;

        \Indp
        \begin{minipage}{\linewidth}
        \begin{tabular}{@{}p{2.2em}p{\dimexpr\linewidth-2.2em\relax}@{}}
        (i)   & Replace $e^{-E_u(\usigma,\utheta^{(t)}_u)}$ in the exact gradient by the degree-$d$ approximation as in \eqref{eq:grad_taylor}.\\
        (ii)  & Expand the resulting moments using \eqref{eq:exp3}.\\
        (iii) & Substitute each monomial expectation $\EX[\prod_{i\in K}\sigma_i]$ by the table entry $S[K]$.\\
        \end{tabular}
        \end{minipage}
        \Indm 
        
        \tcp{Gradient Step and $\ell_1$ projection}
        $\utheta^{(t + 1)}_u \leftarrow  \pc_{\|\cdot\|_1 \le \gamma}(\ \utheta^{(t)}_u - \eta \tilde{\nabla}\lc_{u,S}(\utheta^{(t)}_u) \ )$\;
    }
    
    \tcp{Averaging}
    $\hutheta_u \leftarrow \frac{1}{T} \sum_{t=1}^{T} \utheta^{(t)}_u$\;
}
\Return{$\hat{u}_\theta$}
%\textbf{return} $\hutheta$
% 4. Draw Bottom Line
\vspace{2mm}
\hrule height 1pt
\end{algorithm}

\begin{theorem}
\label{thm:main_thm}
Suppose we want to achieve error $\epsilon \leq 4\gamma$ on the $\theta^*_{u,v}$ parameters in the following metric,
\[
\max_{u\in [p], v \neq u} |\hat{\theta}_{u,v} - \theta_{u,v}^*| \leq \epsilon.
\]
Let $d = \left\lceil \frac{3 \gamma + \log(\frac{16 \gamma(1 + \gamma)}{\epsilon^2})}{ W( \frac{3}{e} + \frac{\log(16 \gamma(1+\gamma)/\epsilon^2)}{\gamma e})} \right\rceil - 1$, $D = d+2-(d\bmod 2)$, and 
\[
\tau = \frac{\epsilon^2 e^{-4\gamma}}{16\,\gamma(1+\gamma)} .
\]
Then Algorithm \ref{alg:ais-moments}, with parameters $T =\left \lceil\frac{576 \ p \gamma^2 e^{8\gamma} (1 + \gamma)^2}{\epsilon^4} \right\rceil$ and $\eta = \frac{2 \gamma e^{-\gamma}}{3 \sqrt{pT}}$, run on a $(D,\tau,\delta)$-approximate moment table $S$ (Definition~\ref{def:moment_table}), returns with probability at least $1 - \delta$ an estimate satisfying the guarantee above simultaneously for every $u \in [p]$. The probability is over the randomness of $S$ only.
\end{theorem}

 The $W$ in the expression for $d$ is the Lambert-$W$ function in the principal branch. From the $\epsilon \leq 4 \gamma$, it can be seen that the argument of this function in the expression for $d$ always remains greater than $3/e$. Since $W$ is an increasing function in the positive real axis, the necessary moment order for this to work scales as $d = O(\gamma)$ for a fixed $\epsilon$. More generally, since the argument of $W$ exceeds $3/e$, the denominator is at least $W(3/e) > 0.6$, and $\log(16\gamma(1+\gamma)) \le \log 16 + 2\gamma$. Hence
\[
d \;\le\; \frac{5\gamma + 2\log(1/\epsilon) + \log 16}{W(3/e)} \;=\; O\big(\gamma + \log(1/\epsilon)\big),
\]
so the order grows linearly with the width and only logarithmically with the inverse of the target accuracy. Numerical experiments in Appendix~\ref{app:experiments} confirm this linear scaling on an example, and show that the degree needed in practice is several times smaller than the one prescribed by the theorem.

The theorem is stated for a generic approximate moment table, so the number of samples needed depends on how the table is produced. Combining it with Lemma~\ref{lem:emp_table} gives the following.

\begin{corollary}
\label{cor:sample_complexity}
Let $d, D, T, \eta$ be as in Theorem~\ref{thm:main_thm} and let $S_{D,n}$ be an empirical table built either from a shared dataset or from separate datasets (Section~\ref{sec:moment_table}), with
\[
n \;\ge\; 2^9\, e^{8 \gamma } \gamma^2 (1 + \gamma)^2\,\frac{(d+2)\log(ep) + \log(\frac{2}{\delta})}{\epsilon^4}.
\]
Then, with probability at least $1-\delta$, Algorithm~\ref{alg:ais-moments} run on $S_{D,n}$ returns $\hat\theta$ with $\max_{u, v \neq u}|\hat{\theta}_{u,v} - \theta_{u,v}^*| \leq \epsilon$.
\end{corollary}
\begin{proof}
With $\tau$ as in Theorem~\ref{thm:main_thm}, $\frac{2}{\tau^2} = \frac{2^9 e^{8\gamma}\gamma^2(1+\gamma)^2}{\epsilon^4}$, and $D\le d+2$. Lemma~\ref{lem:emp_table} then shows that $S_{D,n}$ is a $(D,\tau,\delta)$-approximate moment table, and Theorem~\ref{thm:main_thm} applies.
\end{proof}
The number of samples $n$ used for each moment has the same asymptotic scaling, $O(e^{8\gamma}\mathrm{poly}(\gamma)\frac{\log p}{\epsilon^4})$ for fixed $\epsilon$, as Interaction Screening with full samples in \cite{vuffray2019efficient}. The total measurement cost of the table, however, depends on how it is acquired (see the discussion in Section~\ref{sec:moment_table}): it is $n$ full-sample observations in the shared-dataset model and $n\,p^{O(\gamma)}$ restricted observations in the separate-datasets model.
The total number of steps the algorithm takes is $pT = O(p^2 \gamma^4e^{8\gamma}/\epsilon^4)$, with the dominant cost of each step being $O(|S|)$ time taken to process the moment table given to the algorithm to construct approximate gradients, where $|S| = |\mathcal{K}_D| \le (ep)^D$. The running time is therefore polynomial in the size of the input table. Since $D = O(\gamma)$ for a fixed target accuracy $\epsilon$, it is polynomial in $p$, namely $p^{O(\gamma)}$, whenever the $\ell_1$ width $\gamma$ is a constant independent of $p$. Bounded width is a standard assumption in the literature on learning Ising models \cite{Klivans2017, goel2019learning, vuffray2016interaction}.

\paragraph{Proof of Theorem \ref{thm:main_thm}.}
The main idea of the proof is as follows. Consider the ideal interaction screening loss function given by 
\begin{align} \label{eq:ideal_is}
    \lc_u(\theta) = \EX\left[ e^{-E_u(\usigma, \utheta_u)}\right].
\end{align}
For each $u$, the ideal IS loss function in \eqref{eq:ideal_is} is convex and is known to have its minimum at the true model parameters $\utheta^*$, see \cite{vuffray2016interaction}. Gradient descent for minimizing convex functions has the crucial property of \emph{robustness}, i.e. bounded errors incurred in the gradient computation translate favorably in a controlled way to errors in the final loss function \cite{bubeck_convex_2015} (also see Lemma \ref{lem:robust_gd}). Using this fact and by bounding the error incurred by computing the approximate gradients as in Algorithm~\ref{alg:ais-moments} we can prove Theorem~\ref{thm:main_thm}. % This property of GD is not shared by some of its faster cousins like the GD with momentum. 

The error in the approximate gradient computation is from two sources. The first is a \emph{polynomial approximation error} ($\varepsilon_{poly}$) incurred by approximating the exponential by a finite degree polynomial. The second is a \emph{statistical approximation error} ($\varepsilon_{stat}$) representing the error from replacing the exact moments by the entries of the approximate moment table $S$.  
% requires us to observe full i.i.d samples from the model. However, we will show that by using a polynomial approximation for the exponential function, we can approximate the gradient using only statistics up to $O(\gamma).$% To this end, we first quantify the error induced by approximating the exponential function by using a Taylor polynomial of degree $d$.
%\begin{lemma}[Polynomial approximation for exponential]
% Let $p_d(x) \equiv \sum_{k = 0}^d \frac{(-x)^k}{k!}.$ Then for $d+1 > \gamma$, we can guarantee that,  $\sup_{x \in [-\gamma, \gamma] }|e^{-x} - p_d(x)| \leq e^{-(d+1)(\log(\frac{d+1}{\gamma})-1)}.$
%\label{lem:poly_exp}
%\end{lemma}

We first describe how to obtain the approximate gradients using only moments up to order $d$ from the data. The exact gradients of the loss function $\lc_u$ are,
\begin{align} 
 \dfrac{\partial \lc_{u}(\utheta_u)}{\partial \theta_{u}} &= -\EX\left[\sigma_u e^{-E_u(\usigma, \utheta_u)}\right];  \ \ \dfrac{\partial \lc_{u}(\utheta_u)}{\partial \theta_{uv}} = -\EX\left[ \sigma_u \sigma_v e^{-E_u(\usigma, \utheta_u)} \right],~~ u\neq v, \ \ 
\end{align}
Going forward, we will focus on $\frac{\partial \lc_{u}(\utheta_u)}{\partial \theta_{uv}}$ to demonstrate how the polynomial approximation works. The approximation to $\frac{\partial \lc_{u}(\utheta_u)}{\partial \theta_{u}} $ is mechanically the same. 

We begin by approximating the exponential function  using a polynomial approximation 
\begin{equation} 
\label{eq:grad_taylor}
\EX\left[ \sigma_u \sigma_v e^{-E_u(\usigma, \utheta_u)} \right]
=  \sum_{k \leq d} \dfrac{(-1)^k\EX[\sigma^{k+1}_u \sigma_v  (\sum_{j \neq u} \theta_{uj} \sigma_j + \theta_u)^k ]}{k!} 
+ \varepsilon_{poly}
\end{equation}

Here we have used $\varepsilon_{poly}$  to denote the truncation error. Later in the proof we will set $d$ by tuning this quantity via Lemma \ref{lem:poly_exp}.

Now we will expand the term in the expectation value using the multinomial expansion. For this purpose define,  
$$\mathcal{M}_k \equiv \{ (m_1, \ldots, m_p) \in \mathbb{N}^p | \sum_i m_i = k  \} \qquad  \text{and} \qquad  C_{\utheta, \underline{m}} \equiv \frac{k! \ \ \theta_u^{m_u} \prod_{j \ne u} \theta_{uj}^{\,m_j}}{\prod_{j \in[p]} m_j!}.$$
Using these we can show that,
\begin{equation}
\label{eq:exp3}
\EX[\sigma^{k+1}_u \sigma_v  (\sum_{j \neq u} \theta_{uj} \sigma_j + \theta_u)^k ]  =
\sum_{\underline{m} \in \mathcal{M}_k}  C_{\utheta, \underline{m}} \,\EX\!\left[
\sigma_v \sigma_u^{\,k+1} \prod_{j \ne u} \sigma_j^{\,m_j}
\right]
\end{equation}

This expansion  allows us to estimate the expectation value of a random variable that depends on the optimization parameters by using parameter independent statistics computed from data.  The largest value that $k$ can take is $d$. From \eqref{eq:exp3}, we see that an estimate for $\EX[\sigma^{k+1}_u \sigma_v  (\sum_{j \neq u} \theta_{uj} \sigma_j + \theta_u)^k ]$ for every $k \leq d$ can be constructed entirely using entries of the table $S$ of order $D = d+2-(d\bmod 2)$, which are precisely the entries used to compute the approximate gradient in Algorithm~\ref{alg:ais-moments}.
\begin{align} 
    \widetilde{\nabla}_S \lc(\utheta) = \sum_{k \leq d} \dfrac{(-1)^k\widetilde{\EX}_S[\sigma^{k+1}_u \sigma_v  (\sum_{j \neq u} \theta_{uj} \sigma_j + \theta_u)^k ]}{k!}
    \label{eq:approximate_gradient}
\end{align}
where $\widetilde{\EX}_S[\sigma^{k+1}_u \sigma_v  (\sum_{j \neq u} \theta_{uj} \sigma_j + \theta_u)^k ]$ is the expression on the right-hand side of \eqref{eq:exp3} with each monomial expectation $\EX[\sigma_v \sigma_u^{\,k+1} \prod_{j \ne u} \sigma_j^{\,m_j}]$ replaced by the table entry $S[K]$, where $K\subseteq[p]$ is the set of indices appearing with odd total exponent in the monomial (recall $\sigma_i^2=1$).  As alluded to earlier, the error in this approximate gradient is given by the sum of the polynomial and the statistical approximation errors.
\begin{align} \label{eq:total_gradient_error}
\| \nabla\lc(\utheta) - \widetilde{\nabla}_S \lc(\utheta)\|_\infty \leq \varepsilon_{poly} + \varepsilon_{stat},
\end{align}
where $\varepsilon_{poly}$ is defined by \eqref{eq:grad_taylor} and 
\begin{align} \label{eq:total_statistical_error}
    \varepsilon_{stat} =  \sum_{k \leq d} \dfrac{(-1)^k\widetilde{\EX}_S[\sigma^{k+1}_u \sigma_v  (\sum_{j \neq u} \theta_{uj} \sigma_j + \theta_u)^k ]}{k!} -  \sum_{k \leq d} \dfrac{(-1)^k\EX[\sigma^{k+1}_u \sigma_v  (\sum_{j \neq u} \theta_{uj} \sigma_j + \theta_u)^k ]}{k!}.
\end{align}

The statistical error is bounded using appropriate concentration equalities and is given in the following lemma. The proof can be found in Appendix \ref{proof:estim_err} 

% Now we need to find the error induced in the polynomial expansion by the use of  finite $n$ statistics. This error can be bounded appropriate concentration inequalities. We state the final result here. Proof can be found in 

\begin{lemma}[Statistical approximation error]
\label{lem:finte_sample_1}
Let $S$ be a $(D,\tau,\delta)$-approximate moment table with $D = d+2-(d\bmod 2)$ and $\tau = \varepsilon e^{-\gamma}/(8\gamma)$. Then, with probability at least $1-\delta$, the statistical approximation error satisfies $\varepsilon_{stat} \le \varepsilon/(8\gamma)$ simultaneously for every $u\in[p]$, every coordinate of the gradient, and every $\utheta_u$ in the $\ell_1$ ball of radius $\gamma$.
\end{lemma}

% \begin{lemma}
% \label{lem:finte_sample_1}
% Consider the finite sample estimate for $\sum_{k \leq d} \frac{(-1)^k\EX[\sigma^{k+1}_u \sigma_v  (\sum_{j \neq u} \theta_{uj} \sigma_j + \theta_u)^k ]}{k!}$ constructed using by first expanding it in terms of monomial expectation values as in \eqref{eq:exp3} and then using the observations from $S_{(d+2-(d\bmod 2)),n}$. Then for $n = \left\lceil \frac{(d+2)\log(ep) + \log(\frac{2d+2}{\delta})}{2t^2} \right\rceil,$  the estimation error due to finite $n$ samples is bounded by $te^\gamma$ with probability  greater than $ 1- \delta.$

% \end{lemma}
The next lemma bounds the polynomial approximation error. Proof of this can be found in Appendix \ref{proof:poly_exp}
\begin{lemma}[Polynomial approximation error]
 Let $p_d(x) \equiv \sum_{k = 0}^d \frac{(-x)^k}{k!}.$ If the degree $d$ satisfies 
\begin{align}
d  \geq  \Bigg\lceil \frac{C}{\,W\!\left(\frac{C}{\gamma e}\right)} \Bigg\rceil - 1,
\end{align}
with $C = \log(8\gamma/\varepsilon)$, then we have $\varepsilon_{poly} < \frac{\varepsilon}{8 \gamma}$.
\label{lem:poly_exp}
\end{lemma}

Now we will use the known robustness property of gradient descent to bound how close we can get to the true model by using this approximate gradient. This is proven in Appendix \ref{proof:gd} for completeness.

\begin{lemma}[Robustness of gradient descent]
\label{lem:robust_gd}

Let $\lc : C \to \mathbb{R}$ be a convex function defined over the $\ell_1$ ball $C = \{x\,|\,~~ \|x\|_1 \leq \gamma \}$ with minimizer $x^* \in \arg \min_{x \in C} \lc(x)$. Suppose we have access to an approximate gradient oracle $\widetilde{\nabla} \lc(x)$ for $x \in C$ with error bounded as
\[
\sup_{x \in C} \| \widetilde{\nabla} \lc(x) - \nabla \lc(x) \|_\infty \leq \varepsilon / 4 \gamma.
\]
Let $ \sup_{x \in C} \| \widetilde\nabla \lc(x) \|_2 \leq L$. Now, consider the projected gradient method
\[
x^{t+1} = \mathcal{P}_C \big( x^t - \eta \widetilde{\nabla} \lc(x^t) \big)
\]
starting at $x^1 \in C$ and with fixed step size $\eta
= 2 \gamma/(L \sqrt T)$. Then after $T = \left \lceil \frac{16 \gamma^2 L^2}{\varepsilon^2} \right \rceil$
iterations, the average $\bar{x}^T = \tfrac{1}{T} \sum_{t=1}^T x^t$ satisfies
\[
\lc(\bar{x}^T) - \lc(x^*) \leq \varepsilon.
\]
\end{lemma}

\noindent To apply this lemma we need to bound the norm of the approximate gradient. 
\begin{lemma}[Bounded gradient]
    If $\varepsilon \leq 16 \gamma^2 e^{3\gamma}$, then $\|\widetilde\nabla\lc(\utheta) \|_2 \leq 3\sqrt{p}e^\gamma$.
\end{lemma}
\begin{proof}
For all $\utheta$ in the $\ell_1$ ball or radius $\gamma$
\begin{align}
    \|\widetilde\nabla\lc(\utheta) \|_2 &\leq \sqrt{p}  \|\widetilde\nabla\lc(\utheta) \|_\infty
    \leq \sqrt{p}(\|\nabla\lc(\utheta) \|_{\infty} + \frac{\varepsilon}{4 \gamma}), \nonumber \\ 
    &\leq \sqrt{p}(e^\gamma + \frac{\varepsilon}{4 \gamma}) \leq 3\sqrt{p}e^\gamma, \label{eq:Lbound}
\end{align}
where the last inequality follows from the bound on $\varepsilon$ in the statement of the lemma. 
\end{proof}

Substituting this in $T$ and the learning rate in Lemma \ref{lem:robust_gd} gives us for $T = \frac{144 p \gamma^2 e^{2\gamma}}{\varepsilon^2} $ an estimate $\widehat{\utheta}$ such that,
\begin{equation}
    \lc(\widehat{\utheta}) - \lc(\utheta^*) \leq \varepsilon.
\end{equation}
To obtain the above bound we have used the bounds derived above on $\varepsilon_{poly}$ and $\varepsilon_{stat}$. Lemma~\ref{lem:finte_sample_1} provides the error guarantee on the event that the table $S$ is $\tau$-accurate, which has probability at least $1-\delta$. On this event the bound holds for all iterations $t$ and all nodes $u$ simultaneously, since the table is fixed before the optimization starts and the rest of the argument is deterministic. In particular, no union bound over the $T$ steps or the $p$ nodes is needed.

Now we need to estimate how the error in the loss function propagates to errors in the estimated parameters. To this end, we define a curvature function around the true optimum as follows 
\begin{align}
 \delta \lc^*( \hutheta) &\equiv  \lc(\hutheta) - \lc(\utheta^*) - \langle \nabla \lc(\utheta^*), \hutheta - \utheta^* \rangle 
 \end{align}
 Since the gradient vanishes at $\utheta^*$, we have that, $  \delta \lc^*( \hutheta)=  \lc(\hutheta) - \lc(\utheta^*)  \leq \varepsilon.$ Hence the curvature function is also upper bounded by $\varepsilon$. Independently we can also lower bound the curvature function using a projected infinity norm of the  error vector. We state the main result below, the proof is given in Appendix \ref{app:curv_bound}.
\begin{lemma}\label{lem:strong_conv_exact}
The curvature term is lower-bounded as $\delta \lc^*(\hutheta_u) \geq \frac{e^{-3\gamma}}{ 2 + 2\gamma}  \max_{v \neq  u}  | \theta_{u,v}^* -\hat{\theta}_{u,v}|^2$
\end{lemma}

Putting these together, we get the result, $$\max_{v \neq  u}  | \theta_{u,v}^* -\hat{\theta}_{u,v}| \leq e^{3\gamma/2} \sqrt{2(1 + \gamma) (\lc(\hutheta) - \lc(\utheta^*))} \leq e^{3\gamma/2} \sqrt{2(1 + \gamma) \varepsilon} $$
Now let us define a  error parameter, $\epsilon \equiv \sqrt{\varepsilon} e^{3\gamma/2} \sqrt{2(1 + \gamma)},$ we can guarantee the statement in the theorem with a table accuracy $\tau = \varepsilon e^{-\gamma}/(8\gamma) = \frac{\epsilon^2 e^{-4\gamma}}{16\gamma(1+\gamma)}$, as stated in Theorem~\ref{thm:main_thm}.
\ca
Now the choice of the degree of the approximation was any $d > \lceil \max(Be\gamma, \log(8 \gamma/\varepsilon))/\log(B) \rceil - 1$.

In terms of $\epsilon$ the log term here becomes $\log(\frac{8 \gamma}{\varepsilon}) = \log(\frac{16 \gamma e^{3 \gamma}(1+\gamma) }{\epsilon^2}) =  3 \gamma +  \log(\frac{16 \gamma(1+\gamma) }{\epsilon^2})$

In terms of the final error guarantee $\epsilon$, this implies that for any $B > 1$, then choice $d + 1 > \lceil \max(Be\gamma,\frac {3\gamma +  \log(\frac{16 \gamma(1+\gamma) }{\epsilon^2})}{\log(B)} ) \rceil .$

The condition $\epsilon < 4\gamma$ implies that $\log(16\gamma(1+\gamma)/\epsilon^2) > 1$. So now if we choose $B^*$ to be the solution of the equation $Be - 3/\log(B) = 0$, we get that $d + 1 = \left\lceil B^* \left(e\gamma +   \frac{e}{3}\log(\frac{16 \gamma(1+\gamma) }{\epsilon^2}) \right) \right\rceil.$
\cb

% Returning to \eqref{eq:Lbound}, if we use the relation between $\epsilon$ and $\varepsilon$ and the assumption $\epsilon \leq 4\gamma$ from Theorem~\ref{thm:main_thm}, we can easily show that $L = 3 \sqrt{p}e^{\gamma}$ is a valid upper bound for $\| \tilde \nabla \lc(\utheta) \|.$ This justifies this choice for $L$ that we made right after \eqref{eq:Lbound}.

\qed
\vspace{-0.15in}
\subsection{Structure learning}
In the Ising learning task, recovering the \emph{structure} (i.e. the underlying interaction graph) is important because it determines which pairwise couplings are nonzero and hence which statistical dependencies are present. In many standard formulations, structure recovery is itself a primary goal: once the neighborhood of each node is identified, one can restrict attention to a much smaller set of parameters and thereby reduce both statistical and computational complexity. This perspective is central to several provably efficient approaches that learn Ising models from full samples, including interaction screening and related neighborhood-selection methods \cite{bresler2015efficiently, vuffray2016interaction, vuffray2019efficient}.

In our setting, there is  a more direct motivation, moment-based gradient approximations require combining a large number of low-degree monomial statistics, and the number of relevant monomials can be reduced significantly if we only keep those involving variables in the estimated neighborhood of $u$. Thus, structure learning provides a way to focus the subsequent estimation steps (including magnetic field learning) on the true local interactions, avoiding unnecessary dependence on $p$ in intermediate quantities.

Define the edge set, $E = \{(u,v) \in [p] \times [p] \ | \ \theta^*_{u,v} \neq 0\}.$ Additionally we denote by $\alpha$ a known lower bound on smallest non-zero pair-wise parameter in the model, i.e. $0 <  \alpha \leq \min_{(u,v) \in E} |\theta^*_{u,v}|$.

Given an estimate $\hat\theta$ for the pair-wise parameters, define the thresholded edge set
\begin{equation}
\label{eq:threshold_graph}
\hat{E} \equiv \{(u,v) \in [p] \times [p] ~:~ |\hat{\theta}_{u,v}| \geq \alpha/2\}.
\end{equation}

\begin{theorem}[Structure recovery by thresholding]
\label{thm:struct_learn}
Assume the true model satisfies a separation condition $\min_{(u,v) \in E}|\theta^*_{u,v}| \ge \alpha$ for some known $\alpha>0$. If an estimator $\hat\theta$ satisfies $\max_{u\in [p], v \neq u} |\hat{\theta}_{u,v} - \theta_{u,v}^*| \leq \alpha/2,$
then the thresholded edge set \eqref{eq:threshold_graph} exactly recovers the true structure, i.e.
$\hat{E} = E$.

In particular, if Algorithm \ref{alg:ais-moments} is run with the parameters of Theorem \ref{thm:main_thm} for $\epsilon = \alpha/2$ on a $(D,\tau,\delta)$-approximate moment table, with $D$ and $\tau$ as in Theorem \ref{thm:main_thm}, then thresholding its output recovers $E$ with probability at least $1-\delta$.
\end{theorem}
Since the guarantee of Theorem \ref{thm:main_thm} holds simultaneously for every ordered pair $(u,v)$, the two estimates $\hat\theta_{u,v}$ and $\hat\theta_{v,u}$ of the same coupling fall on the same side of the threshold, so $\hat E$ is symmetric and defines an undirected graph. Note also that $\alpha \le \gamma$, so $\epsilon = \alpha/2$ satisfies the condition $\epsilon \le 4\gamma$ of Theorem \ref{thm:main_thm}.
This is the same standard reduction from parameter learning to structure learning that appears in prior work on interaction screening and related methods (e.g. \cite{vuffray2016interaction}).
\subsection{Learning magnetic fields}
We call the coefficients of the linear terms in \eqref{eq:basic1} \emph{magnetic fields}. Even if they are present in the true model, Theorem \ref{thm:main_thm} does not give a guarantee on their error. This is fundamentally due to the linear term getting swamped by higher order terms in the curvature lower bound (for instance, this happens at \eqref{eq:swamped} in the proof of Lemma \ref{lem:strong_conv}). To get around this issue, we use a re-optimization strategy.

For this we assume that the underlying graphical model has a maximum degree of $D$. After running Algorithm \ref{alg:ais-moments}, we can learn the structure of the model with high probability and fix the second order coefficients. Then we solve a single-variable optimization problem for each variable in the model to estimate the magnetic fields. This can also be done with access to only $O(\gamma)$ statistics. The resulting procedure, Algorithm~\ref{alg:mag-fields} in Appendix~\ref{app:mag_learn_proof}, is Algorithm~\ref{alg:ais-moments} with three changes: for each $u$, the couplings $\theta_{u,v}$ are fixed to their estimates $\hat\theta_{u,v}$; projected gradient descent runs only over the field $\theta_u\in[-\gamma,\gamma]$; and since only the derivative with respect to the field is needed, the approximate gradient is built in the same three steps from a table of order $d+1-(d\bmod 2)$ instead of $d+2-(d\bmod 2)$.

\begin{theorem}[Magnetic field learning]
\label{thm:mag_learn}
Suppose that we have coupling estimates $\hat{\theta}_{u,v}$ such that $|\hat{\theta}_{u,v} -\theta^*_{u,v}| \leq \epsilon$ for all $(u,v) \in E.$ Fix a target accuracy $\epsilon_h > 0$, and let
\[
\varepsilon \coloneqq \frac{\epsilon_h^2 e^{-\gamma}}{2(1+\gamma)},
\qquad
L \coloneqq (1 + D \epsilon) e^{2\gamma} + \frac{\varepsilon}{4 \gamma},
\qquad
T = \left\lceil \frac{16\gamma^2 L^2}{\varepsilon^2} \right\rceil,
\qquad
\eta = \frac{2\gamma}{L\sqrt{T}},
\]
$
 d
= \left\lceil
\dfrac{\,\gamma + \log\!\frac{8\gamma(1+\gamma)}{\epsilon_h^2}\,}
{\,W\!\left(\frac{1}{2e} + \frac{1}{2\gamma e} \log\!\frac{8\gamma(1+\gamma)}{\epsilon_h^2}\right)}
\right\rceil - 1,
$
and
\[
\tau_h \coloneqq \frac{\epsilon_h^2\, e^{-3\gamma}}{8\gamma(1+\gamma)} .
\]
Then Algorithm~\ref{alg:mag-fields}, run with these parameters on a $\big(d+1-(d\bmod 2),\,\tau_h,\,\delta\big)$-approximate moment table, returns with probability at least $1-\delta$ estimates $\hat{\theta}_u$ satisfying, simultaneously for every $u \in [p]$,
\begin{equation}
  |\hat{\theta}_u - \theta_u^*|   \leq \epsilon_h + \sqrt{8\gamma(1+\gamma) e^{3\gamma} D \epsilon}.
\end{equation}
In particular, by Lemma~\ref{lem:emp_table}, an empirical table $S_{d+1-(d\bmod 2),n}$ built from a shared dataset or from separate datasets suffices if
\[
 n \ge 2^{7}\,\gamma^2(1+\gamma)^2 e^{6\gamma}\,\frac{(d+1)\log(ep)+\log\!\left(\frac{2}{\delta}\right)}{\epsilon_h^4} .
\]
\end{theorem}

A key difference compared to Theorem~\ref{thm:main_thm} is that we do not have access to the true couplings $\{\theta^*_{u,v}\}$, and instead must plug in estimates $\{\hat\theta_{u,v}\}$.  This introduces an additional source of error in our gradient oracle, beyond the polynomial truncation and statistical estimation errors. The sparsity of the graph controls how this errors transfers to the gradient estimate.

To quantify this effect, fix $u\in[p]$ and consider these two single-variable loss functions
\begin{align}
    \mathcal{R}(\theta) \equiv \EX \, e^{-\bigl( \theta \sigma_u + \sigma_u\sum_{j \in N_u} \hat{\theta}_{u,j} \sigma_j\bigr)},\ \mathcal{R}^*(\theta) \equiv \EX \, e^{-\bigl( \theta \sigma_u + \sigma_u\sum_{j \in N_u} \theta^*_{u,j} \sigma_j\bigr)}.
\end{align}
Since $\rc^*$ is the right restriction of the complete IS loss function in \ref{eq:ideal_is}, it is clear that $\theta_u^*\in\argmin_{|\theta|\le\gamma}\mathcal{R}^*(\theta)$. Their derivatives satisfy
\begin{align}
  \left| \dfrac{d \rc^*(\theta) }{d \theta} -  \dfrac{d \rc(\theta) }{d \theta}  \right| 
\leq  \EX \left| e^{-\bigl( \theta \sigma_u+ \sigma_u\sum_{j \in N_u} \theta^*_{u,j} \sigma_j\bigr)} -  e^{-\bigl( \theta \sigma_u+ \sigma_u\sum_{j \in N_u} \hat \theta_{u,j} \sigma_j\bigr)}  \right| \\
\leq e^{\gamma} \ \EX \left| e^{- \sigma_u\sum_{j \in N_u} \theta^*_{u,j} \sigma_j} -  e^{-\sigma_u\sum_{j \in N_u} \hat \theta_{u,j} \sigma_j} \right|\\
\leq e^{2\gamma} \EX\left| \sigma_u  \sum_{j \in N_u} \sigma_j (\theta^*_{uj} - \hat \theta_{uj}) \right|  \leq e^{2 \gamma } D \epsilon,  \label{eq:grad_approx}
\end{align}
where the third line uses the basic Lipschitz continuity of the exponential function, $|e^a - e^b| \leq |a-b|\max(e^a,e^b)$. The rest of the proof follows that of Theorem \ref{thm:main_thm}, by considering an approximate optimization strategy for $\rc^*$. The rest of the proof can be found in Appendix \ref{app:mag_learn_proof}.

\subsection{Learning under stronger structural priors}
\label{sec:physics_models}

So far we have looked at cases where the prior information on the model is only a bound on its $\ell_1$ width. It is also interesting to consider cases where even stronger information about the model is known. The main motivation for this comes from studying Gibbs distributions arising in statistical physics and quantum field theory (\cite{bakshi2024learning, shukla2025learning}). There it is often known that the energy function is defined on a $D-$degree graph representing nearest-neighbor spatial points. 

Suppose that the structure of the graph (in terms of the edge set $E$) is known exactly and it is $D$-regular. Then the neighborhood, $|N_u|=D$ for every $u$, and the interaction screening loss at node $u$ depends only on the spins in $\{u\}\cup N_u$:
\begin{align}
\mathcal{L}_{u}(\utheta_u)
&\equiv \EX\Big[ e^{-\sigma_u\left(\theta_u + \sum_{v\in N_u}\theta_{u,v}\sigma_v\right)} \Big].
\end{align}
Define the following function on the $(D+1)$ spins
\begin{equation}
 f_u(\usigma_{\{u\}\cup N_u})
\;\coloneqq\; e^{-\sigma_u\left(\theta_u + \sum_{v\in N_u}\theta_{u,v}\sigma_v\right)},
\qquad \usigma_{\{u\}\cup N_u}\in\{-1,1\}^{D+1}.
\end{equation}
Functions $\{-1,1\}^{D+1}$ admits a multilinear Fourier expansion; $\quad f_u(\usigma_{\{u\}\cup N_u})
= \sum_{K\subseteq \{u\}\cup N_u} \widehat f(\utheta)_{u,K}\prod_{i\in K}\sigma_i.$

%\widehat f_u(K)=2^{-(D+1)}\sum_{\tau\in\{-1,1\}^{D+1}} f_u(\tau)\prod_{i\in K}\tau_i.
Plugging this into the definition of $\mathcal{L}_{u}$ yields, $\mathcal{L}_{u}(\utheta_u)
= \sum_{K\subseteq \{u\}\cup N_u} \widehat f(\utheta)_{u,K}\,\EX\Big[\prod_{i\in K}\sigma_i\Big].$
Hence, evaluating (and differentiating) $\mathcal{L}_{u}$ only requires the moments $m_K$ for the \emph{local} index sets $K\subseteq\{u\}\cup N_u$, which have order $|K|\le D+1$. There are at most $p\,2^{D+1}$ such sets, compared to $|\mathcal{K}_{D+1}| \le (ep)^{D+1}$ entries in a full moment table of the same order. Replacing each $m_K$ by the entry $S[K]$ of an approximate moment table restricted to these local sets, and running projected gradient descent as in Algorithm \ref{alg:ais-moments}, learns all couplings to accuracy $\epsilon$ (Theorem~\ref{thm:known_struct} in Appendix \ref{app:last_theorem}). For an empirical table this requires $n = O\big(2^{2D}e^{8\gamma} \gamma^4 (D + \log(p/\delta))/\epsilon^4\big)$ samples per moment.
\vspace{-0.1in}
\section{Conclusion}
We have shown that having access to $O(\gamma)$ statistics is both computationally and statistically efficient for learning Ising models. As it stands, we see that there is gap in our understanding on the power of statistics below this $O(\gamma)$ threshold but higher than sufficient statistics (order $2$). As we have seen Section \ref{sec:physics_models}, if we assume stronger priors on the model, the learning problem can be potentially solved with even lower order statistics (when $D = o(\gamma)$).  This suggests that restricting the model class in other ways (Ferromagnetic models, Sherrington-Kirkpatrick models, \emph{etc.}) can potentially lower the order of statistics required to solve the learning problem. We leave these extensions for future work.

\begin{ack}
This work has been supported by the U.S. Department of Energy/Office of Science Advanced Scientific Computing Research Program.

\end{ack}

\bibliography{main}
\bibliographystyle{plain}

%%%%%%%%%%%%%%%%%%%%%%%%%%%%%%%%%%%%%%%%%%%%%%%%%%%%%%%%%%%%

\appendix
\section{Technical Lemmas}

\subsection{Proof of Lemma \ref{lem:poly_exp}}
\label{proof:poly_exp}
\begin{proof}

Let
\[
P_d(x)=\sum_{k=0}^{d}\frac{(-x)^k}{k!}
\]
be the degree-$d$ Taylor polynomial of \(e^{-x}\) about \(0\). We want a uniform bound on the error
\[
E_d(\gamma):=\max_{x\in[-\gamma,\gamma]}\big|e^{-x}-P_d(x)\big|.
\]

We have,
\[
\big|e^{-x}-P_d(x)\big|\le\sum_{k=d+1}^{\infty}\frac{|x|^k}{k!} \leq \sum_{k=d+1}^{\infty}\frac{\gamma^k}{k!}
\]

Then,
\[
\sum_{k=d+1}^{\infty}\frac{\gamma^k}{k!}
= e^{\gamma}\Big(1-e^{-\gamma}\sum_{k=0}^{d}\frac{\gamma^k}{k!}\Big)
= e^{\gamma}\Pr\big[\mathrm{Poisson}(\gamma)\ge d+1\big].
\]

Therefore,
\[
E_d(\gamma)\le e^{\gamma}\Pr\big[\mathrm{Poisson}(\gamma)\ge d+1\big].
\]

Use the standard Poisson (Chernoff) tail bound: for \(a>\gamma\),
\[
\Pr\big[\mathrm{Poisson}(\gamma)\ge a\big]\le e^{-\gamma}\Big(\frac{e\gamma}{a}\Big)^a.
\]
Taking \(a=d+1\)  yields
\[
E_d(\gamma)\le e^{\gamma}\cdot e^{-\gamma}\Big(\frac{e\gamma}{d+1}\Big)^{d+1}
= \Big(\frac{e\gamma}{d+1}\Big)^{d+1} = e^{(d+1)(\log(\frac{\gamma}{d+1}) + 1)}
\]

%Now, for any $B > 1$, if we enforce $d+1  \geq B e \gamma $, we obtain $E_d(\gamma) \leq e^{-(d+1)\log(B)}. $ From this, a second condition $d + 1 \geq \log(1/\eta)/\log(B)$ ensures that the error is always upper bounded by $\eta$. Taking the maximum of the two conditions on $d+1$, and rounding up to the next largest integer, gives us the result in the theorem.

To finish the proof, we set $\varepsilon_{poly} = \frac{\varepsilon}{8 \gamma}$.  From this we get the following implicit lower bound for $d$ that ensures that the error remains small, $ e^{-(d+1)(\log(\frac{d+1}{\gamma})-1)} \leq   \frac{\varepsilon}{8 \gamma}$. The implicit inversion can be performed by a simple yet tedious calculation and is detailed in Appendix \ref{app:lW_calculation}.

% Using a simple but tedious calculation  (detailed in Appendix \ref{app:lW_calculation}) the smallest positive integer satisfying this condition can be expressed in terms of the Lambert-$W$ function,
%     $d = \left\lceil \frac{3 \gamma + \log(\frac{16 \gamma(1 + \gamma)}{\epsilon^2})}{ W( \frac{3}{e} + \frac{\log(16 \gamma(1+\gamma)/\epsilon^2)}{\gamma e})} \right\rceil - 1 $.

\end{proof}

\subsection{Proof of Lemma \ref{lem:finte_sample_1}}
\label{proof:estim_err}

Let $S$ be a $(D,\tau,\delta)$-approximate moment table, and work on the event $\mathcal{E}$ on which $\max_{K\in\mathcal{K}_D}|S[K]-m_K|\le\tau$, which has probability at least $1-\delta$ by Definition~\ref{def:moment_table}.

Using the table entries in \eqref{eq:exp3}, we get a random approximation for the expectation value in the LHS, which we call $\widetilde{\EX}_S[\sigma^{k+1}_u \sigma_v  (\sum_{j \neq u} \theta_{uj} \sigma_j + \theta_u)^k ]$. We use the $\widetilde{\EX}_S$ here to emphasize that this approximation is constructed by combining table entries, and is not itself a moment of the model.

Since the $\sigma$ variables square to $1$, each monomial $\sigma_v \sigma_u^{\,k+1} \prod_{j \ne u} \sigma_j^{\,m_j}$ equals $\prod_{i\in K}\sigma_i$ for a set $K$ with $|K|\le k+2$, and for $k\le d$ we have $|K|\le D$, so the corresponding table entry $S[K]$ exists. From $\eqref{eq:exp3}$, we compute
\begin{align*}
 \sum_{\underline{m} \in \mathcal{M}_k}
|C_{\utheta, \underline{m}}|  = (\sum_{j \neq u} |\theta_{uj}|+ |\theta_u|)^k 
\le \gamma^k
\end{align*}
for every $\utheta_u$ in the $\ell_1$ ball of radius $\gamma$. On $\mathcal{E}$ this implies that,
\begin{align*}
\left|\EX[\sigma^{k+1}_u \sigma_v  (\sum_{j \neq u} \theta_{uj} \sigma_j + \theta_u)^k ] - \widetilde{\EX}_S[\sigma^{k+1}_u \sigma_v  (\sum_{j \neq u} \theta_{uj} \sigma_j + \theta_u)^k ] \right|
\leq \gamma^k \tau .
\end{align*}
Using this for each $k \leq  d$ in the expression in \eqref{eq:grad_taylor}, we get on $\mathcal{E}$,
\begin{align} 
\varepsilon_{stat}  \leq \tau\sum_{k \leq d} \frac{\gamma^k}{k!} \leq \tau e^\gamma .
\end{align}
We set $\tau e^\gamma = \varepsilon/(8\gamma)$, i.e. $\tau = \varepsilon e^{-\gamma}/(8\gamma)$, which gives the claim. The bound holds for every coordinate of the gradient, every $u$ and every $\utheta_u$ on the same event $\mathcal{E}$, so no further union bound is needed.
% gives us for   $n = \left\lceil 32\ e^{2 \gamma } \ \gamma^2\frac{(d+2)\log(ep) + \log(\frac{2pT(d+1)}{\delta})}{\varepsilon^2} \right\rceil,$ with probability $1-\delta/T$ , an $\ell_{\infty}$ approximation for the gradient of the loss function (denoted by $\widetilde{\nabla}_n \lc(\utheta)$), constructed by appropriately combining estimates from $S_{d+2-(d\bmod 2), \ n}$, satisfying,
% \begin{equation}
% \label{eq:err_grad}
%  \| \nabla\lc(\utheta) - \widetilde{\nabla}_n \lc(\utheta)\|_\infty \leq \frac{\varepsilon}{4\gamma}
% \end{equation}

\ca
\begin{proof}

Lets first approximate $\exp(\gamma x) $  in $[-1,1]$ using Chebyshev polynomials. It is elementary to show that the coefficients of this expansion correspond to the modified Bessel functions of the first kind  \cite{DLMF-10.35}. These functions decay super-exponentially in their order. This allows for very fast convergence of the series expansion.

\[
e^{\gamma y} =  I_{0}(\gamma) + 2\sum_{n=1}^{\infty} I_{n}(\gamma)\,T_{n}(y)
\]

Hence we see that $I_{0}(\gamma) + 2\sum_{n=1}^{d} I_{n}(\gamma)\,T_{n}(x) = p_d(-\gamma y)$ is a plausible polynomial approximation for $e^{\gamma y}$

%Since $I_k(\gamma) = \frac{1}{\pi}\int^\pi_0 exp(\gamma \cos(\theta)) cos(k \theta) d \theta = \frac{1}{\pi} \int_{-1}^1 \frac{exp(\gamma y) T_k(y)}{\sqrt{1 - y^2}} dy$

It is sufficient to bound the error in approximating $\exp(\gamma y)$,
$$\max_{x \in [-\gamma, \gamma] }|e^{-x} - p_d(x)| =  \max_{x \in [-\gamma, \gamma] }|e^{x} - p_d(-x)| =  \max_{y \in [-1, 1] }|e^{\gamma y} - p_d(-\gamma y)|$$,

Now from the exact expansion for the exponential,
\begin{equation}\label{eq:Cheb_bound1}
\max_{y \in [-1, 1] }|e^{\gamma y} - p_d(-\gamma y)|  \leq  2\sum_{n > d} |I_n(\gamma)|  
\end{equation}

Now,
\begin{equation}
    I_n(\gamma) = \sum_{k = 0}^\infty \frac{(\gamma/2)^{2k + n}}{k! (n+k)!} \leq  (\gamma/2)^n \sum_{k = 0}^\infty \frac{(\gamma/2)^{2k}}{(k!)^2 n! } = I_0(\gamma) \frac{(\gamma/2)^n}{n!} =\frac{1}{\pi}\left(\int_0^\pi d \theta e^{\gamma \cos(\theta)}\right) \frac{(\gamma/2)^n}{n!} \leq e^{\gamma}\frac{(\gamma/2)^n}{n!}
\end{equation}

Now  using this in the error bound,
$$ \sum_{n > d} |I_n(\gamma)|  \leq e^{\gamma} \sum_{n >d} \frac{(\gamma/2)^n}{n!} \leq e^{\frac{3 \gamma}{2}} Pr[Poisson(\gamma/2) > d]$$

%Now from a Chernoff bound \cite{Canonne2017PoissonConcentration}, $Pr[Poisson(\gamma/2) > \gamma/2 + x] < \exp(\frac{x^2}{(\gamma/2 + x)})$
Now from a standard Chernoff tail-bound for the Poisson tail we have for $d >  \gamma/2$, $$  Pr[Poisson(\gamma/2) > d] \leq e^{-\gamma/2} (\frac{e \gamma}{d+1})^{d+1} $$ 

Now plugging this back into \eqref{eq:Cheb_bound1} gives us,
$$\max_{x \in [-\gamma, \gamma] }|e^{-x} - p_d(x)| \leq 2 e^{\gamma} (\frac{e \gamma}{d+1})^{d+1}  $$

Now let's take that $d + 1 \geq {2e\gamma},$ implying that the approximation error is upper bounded by $ 2 e^{\gamma - (d+1)\log(2)}$.

Now if we choose  $(d+1) \geq  \frac{1}{\log(2)}(\gamma + \log(\frac2\eta)) $.

Combining the three lower bounds introduced in this proof we get $d \geq 2e\gamma + \frac{\log(2/\eta)}{\log(2)}$

\end{proof}
\ca
\section{$\sqrt{\gamma}$ approximation}
Let us start with some defintions
\begin{itemize}
    \item $X_i$ are i.i.d  Rademacher random variables and let $Y_t \equiv \sum_{i = 1}^t X_t$ be the postion of a 1D random walker starting from origin. 
    \item $\mathbb{I}_D: \mathbb{R} \rightarrow \{ 0,1\}$ be the indicator function that is one if and only if its input lies in $[-D,D]$ 
    \item $T_k$ are the $k$-th Chebyshev polynomials and we use the convention $T_{-k}(x) = T_{k}$
\end{itemize}

Now it can be shown that the $d$-degree polynomial $p_{t,d}(x) = \EXp{Y_t}~T_{Y_t}(x) \mathbb{I}_d(Y_t)$ is a good polynomial approximation to $x^t$ for $d$ that scales are $\sqrt{t}$. We state the exact result below. The proof can be found in the monograph of Sachdeva and Vishnoi \cite{sachdeva_faster_2013}.

\begin{lemma}(Sachdeva and Vishnoi \cite{sachdeva_faster_2013})

For any positive integers $t$ and $d$, the degree-$d$ polynomial $p_{t,d}$ defined by Equation~(3.1) satisfies
\[
\sup_{x\in[-1,1]} \left| p_{t,d}(x) - x^{t} \right| \le 2e^{-d^{2}/(2t)}.
\]
Hence, for any $\delta>0$, if
\[
d = \left\lceil \sqrt{2t\,\log\!\left(\frac{2}{\delta}\right)} \right\rceil,
\]
then
\[
\sup_{x\in[-1,1]} \left| p_{t,d}(x) - x^{t} \right| \le \delta.
\]

\end{lemma}

Now we can use this result to get a polynomial approximation for the exponential function that scales as the square root of the $l_1$ width. This is a modification of the results in \cite{sachdeva_faster_2013}.

\begin{lemma}(Quadratically improved approximation
\end{lemma}
\cb

\section{Proof of Lemma \ref{lem:robust_gd}}
\label{proof:gd}

This is a modification of the proof found in \cite{bresler_hardness_2014} to the $\ell_1$ ball.
Define the un-projected intermediate step; $y^{t+1} = x^t - \eta \widetilde{\nabla} \lc(x^t)$

\begin{align*}
\lc(x^t) - \lc(x^*) 
&\le \nabla \lc(x^t)^\top (x^t - x^*) \\
&= \widetilde{\nabla} \lc(x^t)^\top (x^t - x^*) 
  + \big(\nabla \lc(x^t)^\top - \widetilde{\nabla}\lc(x^t)^\top\big)(x^t - x^*) \\
&\le \widetilde{\nabla} \lc(x^t)^\top (x^t - x^*) + \varepsilon\,\|x^t - x^*\|_1/4\gamma \\
&\leq \frac{1}{\eta}\,(x^t - y^{t+1})^\top (x^t - x^*) + \varepsilon/2\\
&= \frac{1}{2\eta}\!\left(\|x^t - x^*\|_2^2 + \|x^t - y^{t+1}\|_2^2 - \|y^{t+1} - x^*\|_2^2\right) +\varepsilon/2 \\
&= \frac{1}{2\eta}\!\left(\|x^t - x^*\|_2^2 - \|y^{t+1} - x^*\|_2^2\right)
   + \frac{\eta}{2}\,\|\widetilde{\nabla}\lc(x^t)\|_2^2 +  \varepsilon/2 \, .
\end{align*}

The projection to a convex body always contracts (see Lemma 3.1 in \cite{bubeck_convex_2015}), hence,
$ \|y^{t+1} - x^*\|_2 \geq \| \mathcal{P}_C(y^{t+1}) - \mathcal{P}_C(x^*)\|_2  = \|x^{t+1} - x^*\|_2.$ . Using this we get  we get following inequality,
\begin{equation}
   \lc(x^t) - \lc(x^*)  \leq \frac{1}{2\eta}  \!\left(\|x^t - x^*\|_2^2 - \|x^{t+1} - x^*\|_2^2\right)
   + \frac{\eta}{2}\,\|\widetilde{\nabla}\lc(x^t)\|_2^2 +  \varepsilon/2 \, 
\end{equation}

Now summing these from $1$ to $T$ and using the facts that $||x^1 - x^*||_2 \leq 2 \gamma$ and $\|\widetilde{\nabla}\lc(x^t)\|_2^2 \leq L,$ we get,

\begin{align}
\sum_{t = 1}^{T}(\lc(x^t) - \lc(x^*)) \leq \frac{1}{\eta} (2 \gamma^2) + \frac{\eta T L^2}{2} +  T \varepsilon/2 .
\end{align}

By fixing $\eta = \frac{2 \gamma}{L \sqrt{T}}$ then gives us,
\begin{align}
\frac1T \sum_{t = 1}^{T}\lc(x^t) - \lc(x^*) \leq \frac{2 \gamma L}{ \sqrt{T}} +  \varepsilon/2 .
\end{align}

Now we know from convexity that  $ \lc(\frac1T \sum_{t = 1}^{T}x^t) -\lc(x^*) \leq \frac1T \sum_{t = 1}^{T}\lc(x^t) -\lc(x^*).$

Now choosing  $T = \left\lceil \frac{16 \gamma^2 L^2}{\varepsilon^2} \right \rceil$ gives the desired result

\section{Curvature bounds}
\label{app:curv_bound}

\paragraph{Notation}
\begin{enumerate}
    \item  $\utheta_u^*$, true parameters connected to the variable u
    \item $\hat{\lc}_{n}$,  IS loss for the variable $u$, estimated from $n$ samples
    \item $\tilde{\utheta}_u = \argmin{\|\utheta \|_1 \leq \gamma} \hat{\lc}(\utheta)_{n}$
    \item $\hat{\theta}_u$, an approximate minimizer of the loss $\hat{\lc}_{n}$
\end{enumerate}

\begin{lemma}\label{lem:strong_conv}
 The curvature term is lower bounded as $\delta \lc^*(\hutheta_u) \geq \frac{e^{-3\gamma}}{ 2 + 2\gamma} (\max_{u \neq v}|\theta_{u,v}^* -\hat\theta_{u,v}|^2) $
\end{lemma}

\begin{proof}

We work with the definition of the curvature function,
$$\delta \lc^*( \utheta) \equiv  \lc(\utheta) - \lc(\utheta^*) - \langle \nabla \lc(\utheta^*), \utheta - \utheta^* \rangle $$
 Define local energy function $E_u(\usigma; \theta) = \sum_{k \neq u} \theta_{u,k}\sigma_u\sigma_k + \theta_u \sigma_u $. Notice that this function is linear in the $\theta$ variables.  So given an error parameter $\Delta \equiv   \theta- \theta^*$ on the parameters we can see that,

\begin{align}
   \lc(\utheta) = \lc(\utheta^* + \uDelta)  =  \Emu \exp(-E_u(\usigma;\utheta^* + \uDelta))  =  \Emu \exp(-E_u(\usigma;\utheta^*)) \  \exp( - E_u(\usigma;\uDelta)). 
\end{align}
The following is also true,
\begin{align}
    \langle \nabla \lc(\utheta^*), \uDelta \rangle &= \Emu 
    -\exp(-E_u(\usigma;\utheta^*))(\sum_{k \neq u} \dfrac{\partial E_u(\usigma;\utheta_u)}{\partial \theta_{uk}} \Delta_{uk} + \dfrac{\partial E_u(\usigma;\utheta_u)}{\partial \theta_{u}} \Delta_u ) \\ 
    &= -\Emu \exp(-E_u(\usigma;\utheta^*)) E_u(\usigma;\uDelta).
\end{align}

The last two expressions together imply that given a function $g(x) = e^{-x} - 1 + x$, we can write the curvature term as,

\begin{align} \label{eq:curv1}
    \delta \lc^*(\utheta) = \Emu \exp(-E_u(\usigma, \utheta^*))  g(E_u(\usigma, \uDelta_u))
\end{align}

First observe that $g(x) \geq  \frac{x^2}{2 + |x|}$ for all $x \in \mathbb{R}$ (See Lemma \ref{lem:notes2} for the proof)

Using this and the fact that $|E_u(\usigma, \utheta^*)| \leq \gamma$ we get,

\begin{equation}\label{eq:curv2}
    \delta \lc^*(\utheta_u) \geq  \frac{e^{-\gamma}}{2 + 2\gamma }~\Emu (E(\usigma,\uDelta))^2
\end{equation}

We lower bound the above in terms of the variance,
\begin{align}\label{eq:swamped}
    \Emu (E_u(\usigma,\uDelta))^2 &= \Emu (\sum_{k\neq u} \sigma_k \Delta_{u,k} + \Delta_u) ^2, \\
    &\geq \Var{\usigma} \sum_{k\neq u} \sigma_k \Delta_{u,k}.
\end{align}
Note that the price to pay here is we lose control over the error on the magnetic field term.

Now we can use the law of variances to bound the expectation value. For any $v \neq u$ we can do the following,

\begin{align}
    \Var{\usigma} \sum_{k\neq u} \sigma_k \Delta_{u,k},  & \geq \EXp{\usigma_{\setminus v}} \Var{}\left[ \sum_{k\neq u} \sigma_k \Delta_{u,k} \ \bigg| \ \usigma_{\setminus v}\right]\\
    &= \Delta^2_{u,v} \EXp{\usigma_{\setminus v}} \Var{}[ \sigma_v| \usigma_{\setminus v}]
\end{align}

Now it is well known that conditional variances of a single variable Ising model can be lower bounded by a function of the $\ell_1$-width (\cite{Klivans2017, vuffray2016interaction})

\begin{align}
\Var{\sigma_v \sim \mu(.|\usigma_{\setminus v})}~[ \sigma_v| \usigma_{\setminus v}] = 1 - \tanh^2( \sum_{j \neq v } \theta^*_{vj} \sigma_j + \theta^*_v) \geq 1 - \tanh^2(\gamma) \geq e^{-2 \gamma}.
\end{align}

Using this and maximizing over all $v \neq u$ we get the following relation
\begin{equation}
 \delta \lc^*(\utheta)\geq  \frac{e^{-\gamma}}{2 + 2\gamma }~\Emu (E_u(\usigma,\uDelta))^2 \geq \frac{e^{-3\gamma}}{ 2 + 2\gamma} \max_{v \neq u} |\theta_{uv} - \theta_{uv}^*|^2
\end{equation}
\end{proof}

\section{Bounding \texorpdfstring{$d$}{d}}\label{app:lW_calculation}

\begin{align}
& e^{-(d+1)\left(\log\!\frac{d+1}{\gamma}-1\right)}
   \le e^{-C} \\[2mm]
\iff\;& (d+1)\!\left(\log\!\frac{d+1}{\gamma}-1\right)
        \ge C. \label{eq:Cdef}
\end{align}

% \begin{align}
% & e^{-(d+1)\left(\log\!\frac{d+1}{\gamma}-1\right)}
%    \le \frac{\epsilon^{2} e^{-3\gamma}}{16(1+\gamma)\gamma}
%    \label{eq:ineq_d} \\[2mm]
% \iff\;& -(d+1)\!\left(\log\!\frac{d+1}{\gamma}-1\right)
%         \le \log\!\frac{\epsilon^{2}}{16(1+\gamma)\gamma}-3\gamma \\[1mm]
% \iff\;& (d+1)\!\left(\log\!\frac{d+1}{\gamma}-1\right)
%         \ge 3\gamma+\log\!\frac{16(1+\gamma)\gamma}{\epsilon^{2}}
%         \;=:\; C. \label{eq:Cdef}
% \end{align}

Let \(x:=d+1\). Since \(f(x)=x(\log(x/\gamma)-1)\) is increasing for \(x>\gamma\),
the least integer \(x\) satisfying \eqref{eq:Cdef} is the ceiling of the
solution of the equation \(x\!\left(\log(x/\gamma)-1\right)=C\).
Let's solve this equation:
\begin{align}
x\!\left(\log\!\frac{x}{\gamma}-1\right)&=C \\\implies 
x\,e^{-C/x}&=\gamma e.
\end{align}
Set \(t:=C/x\). Then
\begin{align}
\frac{C}{t}\,e^{-t}=\gamma e
\quad\Longrightarrow\quad
t\,e^{t}=\frac{C}{\gamma e}.
\end{align}
The $W$ function enters here. Using the defining identity for the $W$ function, \(t\,e^{t}=z \iff t=W(z)\),
\begin{align}
t\,e^{t}=\frac{C}{\gamma e}
\quad\Longrightarrow\quad
t=W\!\left(\frac{C}{\gamma e}\right)
\quad\Longrightarrow\quad
\frac{C}{x}=W\!\left(\frac{C}{\gamma e}\right).
\end{align}
Hence
\begin{align}
x=\frac{C}{\,W\!\left(\frac{C}{\gamma e}\right)}\!,
\qquad
C=3\gamma+\log\!\frac{16(1+\gamma)\gamma}{\epsilon^{2}}.
\end{align}
Therefore the smallest integer \(d\) satisfying \eqref{eq:Cdef} is
\begin{align}
d
= \Bigg\lceil \frac{C}{\,W\!\left(\frac{C}{\gamma e}\right)} \Bigg\rceil - 1
= \Bigg\lceil
\frac{\,3\gamma+\log\!\dfrac{16(1+\gamma)\gamma}{\epsilon^{2}}\,}
{\,W\!\left(\dfrac{3}{e}+\dfrac{1}{\gamma e}
\log\!\dfrac{16(1+\gamma)\gamma}{\epsilon^{2}}\right)}
\Bigg\rceil - 1.
\end{align}
\noindent
(Here \(W\) denotes the principal branch when the argument is positive.)

\section{Notes}

\begin{lemma}[Poisson Chernoff tail bound] For $a > b \in \mathbb{R}$,
$\Pr[Poisson(b) \geq a] \leq e^{-b}(\frac{eb}{a})^a$
\end{lemma}

\begin{proof}
Let $X \sim \text{Poisson}(b)$, then $\Pr[X = k] = \frac{e^{-b}b^k}{k!}$
.We will do a Chernoff bound to bound the tail. For any $t >0,$ from the Markov inequality we have,
\begin{equation}
   \Pr[X \geq a] = \Pr[e^{tX} \geq e^{ta}]  \leq \frac{\mathbb{E}e^{tX}}{e^{ta}}
\end{equation}

Now $\mathbb{E}e^{tX} = e^{-b}\sum_{k = 0}^\infty \frac{(e^t b)^k}{k!} = e^{e^tb - b} $. Hence for all $t > 0$, we have $\Pr[X \geq  a] \leq  e^{e^tb -b - ta}.$ 

Now optimium $t$ for this falls at $\log(\frac{a}{b}).$ Plugging this $t$ into the tail bound above gives the desired answer.
    
\end{proof}

\begin{lemma}
\label{lem:notes2}
For all $x\in\mathbb{R}$, the function $g(x)\coloneqq e^{-x}-1+x$ satisfies
\[
g(x)\;\ge\;\frac{x^{2}}{\,2+|x|\,}.
\]
\end{lemma}

\begin{proof}
Define the auxiliary function
\[
M(x)\;\coloneqq\;(2+|x|)\,g(x)-x^{2}
\;=\;(2+|x|)\,(e^{-x}-1+x)-x^{2}.
\]
Clearly $M(0)=0$. We show $M(x)\ge 0$ for all $x$, which is equivalent to the desired inequality.

\medskip
\noindent\textbf{Case 1: $x\ge 0$.}
Here $|x|=x$, so
\[
M(x)=(2+x)(e^{-x}-1+x)-x^{2}.
\]
Differentiating,
\[
M'(x)=1-(1+x)e^{-x},\qquad
M''(x)=x\,e^{-x}>0\quad(x>0).
\]
Hence $M'$ is increasing on $[0,\infty)$ with $M'(0)=0$, so $M'(x)\ge 0$ for $x\ge 0$ and $M$ is nondecreasing on $[0,\infty)$. Since $M(0)=0$, we get $M(x)\ge 0$ for all $x\ge 0$.

\medskip
\noindent\textbf{Case 2: $x\le 0$.}
Here $|x|=-x$, so
\[
M(x)=(2-x)(e^{-x}-1+x)-x^{2}.
\]
Differentiating,
\[
M'(x)=3(1-e^{-x})-4x+x\,e^{-x},\qquad
M''(x)=4(e^{-x}-1)-x\,e^{-x}.
\]
For $x<0$ we have $e^{-x}-1>0$ and $-x>0$, hence $M''(x)>0$. Thus $M'$ is increasing on $(-\infty,0]$ and $\lim_{x\to 0^-}M'(x)=0$, which implies $M'(x)\le 0$ for $x\le 0$. Therefore $M$ is nonincreasing on $(-\infty,0]$, so for every $x\le 0$ we have $M(x)\ge M(0)=0$.

\medskip
Combining the two cases gives $M(x)\ge 0$ for all $x\in\mathbb{R}$, i.e.
\[
(2+|x|)\bigl(e^{-x}-1+x\bigr)\;\ge\;x^{2},
\]
which is equivalent to $g(x)\ge x^{2}/(2+|x|)$. Equality holds at $x=0$.
\end{proof}

\section{Magnetic field learning}\label{app:mag_learn_proof}

\begin{algorithm} [t]
% 1. Draw Top Line
\hrule height 1pt
\vspace{1mm}

% 2. Caption goes here
\caption{Magnetic Field Learning with Moment-Based Approximate Gradients}
\label{alg:mag-fields}

% 3. Draw Separator Line
\vspace{1mm}
\hrule height 0.4pt
\vspace{2mm}

\SetAlgoLined
\DontPrintSemicolon

\textbf{Input:} Degree $d$, coupling estimates $\{\hat\theta_{u,v}\}$,
approximate moment table $S$ of order $d + 1 - (d \bmod 2)$, width bound $\gamma$\;, step size $\eta$, iterations $T$\;

\textbf{Output:} Magnetic field estimates $\{\hat\theta_u\}_{u\in[p]}$\;

\BlankLine

\For{$u = 1$ \textbf{to} $p$}{
    Initialize $\theta_u^{(1)} \leftarrow 0$\;

    \For{$t = 1$ \textbf{to} $T$}{
        \tcp{Gradient Approximation}
        Construct $\tilde\nabla\rc_u(\theta_u^{(t)})$ from $S$ as follows:\;

        \Indp
        \begin{minipage}{\linewidth}
        \begin{tabular}{@{}p{2.2em}p{\dimexpr\linewidth-2.2em\relax}@{}}
        (i)   & Replace the exponential in the exact gradient by a degree-$d$ approximation as in \eqref{eq:grad_taylor_h}.\\
        (ii)  & Expand the resulting moments using \eqref{eq:exp3_h}.\\
        (iii) & Substitute each monomial expectation $\EX[\prod_{i\in K}\sigma_i]$ by the table entry $S[K]$.\\
        \end{tabular}
        \end{minipage}
        \Indm

        \tcp{Gradient Step and projection}
        $\theta_u^{(t+1)} \leftarrow \pc_{|\cdot|\le \gamma}\big(\theta_u^{(t)} - \eta\,\tilde\nabla\rc_u(\theta_u^{(t)})\big)$\;
    }

    \tcp{Averaging}
    $\hat\theta_u \leftarrow \frac{1}{T}\sum_{t=1}^{T}\theta_u^{(t)}$\;
}

\Return{$\{\hat\theta_u\}_{u\in[p]}$}

% 4. Draw Bottom Line
\vspace{2mm}
\hrule height 1pt
\end{algorithm}

\subsection{Proof of Theorem \ref{thm:mag_learn}}

\begin{proof}

Now use the same strategy as in Theorem \ref{thm:main_thm} to approximate the gradient of $\rc(\theta)$ and expand exponential in the derivative up to order $d$. 

\begin{equation} \label{eq:grad_taylor_h}
\EX\left[ \sigma_u e^{-( \theta \sigma_u+ \sigma_u\sum_{j \in N_u} \hat{\theta}_{uj} \sigma_j)} \right]
=  \sum_{k \leq d} \dfrac{(-1)^k\EX[\sigma^{k+1}_u   (\sum_{j \in N_u} \hat \theta_{uj} \sigma_j + \theta)^k ]}{k!} 
+ \varepsilon_{poly}
\end{equation}

Now the powers in this expansion can be expanded using multinomial coefficients as in \eqref{eq:exp3}

\begin{equation}
\EX[\sigma^{k+1}_u (\sum_{j \in N_u} \hat{\theta}_{uj} \sigma_j + \theta)^k ] =
\sum_{\underline{m} \in \mathcal{M}_k}  \hat{C}_{\theta, \underline{m}} \,\EX\!\left[
\sigma_u^{\,k+1} \prod_{j \ne u} \sigma_j^{\,m_j}
\right].
\label{eq:exp3_h}
\end{equation}

Where $\mathcal{M}_k$ is defined as in \eqref{eq:exp3} and $\hat{C}_{\theta, \underline{m}}   \equiv \frac{k! \ \ \theta^{m_u} \prod_{j \ne u} \hat \theta_{uj}^{\,m_j}}{\prod_{j \in[p]} m_j!}.  $ Here we have used the convention that $\hat \theta_{u,j} = 0$ for $j \notin N_u.$

Since the largest value of $k$ is $d$, we see that every monomial in \eqref{eq:exp3_h} corresponds to an entry of a moment table of order $d+1-(d\bmod 2)$. Replacing each monomial expectation by the corresponding table entry gives the following approximation for the derivative of $\mathcal{R}$ w.r.t $\theta$.

\begin{lemma}
\label{lem:finite_sample_2}
Let $S$ be a $\big(d+1-(d\bmod 2),\tau,\delta\big)$-approximate moment table. Consider the estimate for $\sum_{k \leq d} \dfrac{(-1)^k\EX[\sigma^{k+1}_u   (\sum_{j \in N_u} \hat \theta_{uj} \sigma_j + \theta)^k ]}{k!} $ constructed by first expanding it in terms of monomial expectation values as in \eqref{eq:exp3_h} and then replacing each of them by the corresponding entry of $S$. Then, with probability at least $1-\delta$, the estimation error is bounded by $\tau e^{2\gamma}$ simultaneously for every $u\in[p]$ and every $|\theta|\le\gamma$.

\end{lemma}
Putting the polynomial approximation together with \eqref{eq:grad_approx}, gives us a approximation for gradient of $\rc(\theta)$. We call this approximation $\tilde \nabla \rc(\theta)$. Totally we have  with probability  $1 - \delta,$

\begin{equation}
\left|    \dfrac{d \rc^*(\theta)}{d \theta} - \tilde \nabla \rc(\theta) \right| \leq e^{2 \gamma} D \epsilon + \tau e^{2\gamma} + \varepsilon_{poly}
\end{equation}

Now to keep things simpler, we define a new error term $\varepsilon $ and choose the table accuracy $\tau$ and $\varepsilon_{poly}$ such that $\tau e^{2\gamma} = \varepsilon_{poly} = \frac{\varepsilon}{4\gamma}$.

Now we can use this approximate gradient in Lemma \ref{lem:robust_gd} to bound how close we can get to $\rc^*(\theta^*_u)$. To this end we need to bound the $\ell_1$ norm (just the absolute value in this case) of the approximate gradient.
\begin{equation}
   |\tilde{\nabla} \rc(\theta)|  \leq \left|\dfrac{d \rc^*(\theta)}{d \theta} \right| + e^{2 \gamma} D \epsilon + \frac{\varepsilon}{4\gamma} \leq (1 + D \epsilon) e^{2\gamma} + \frac{\varepsilon}{4 \gamma}.
\end{equation}

Choose
\begin{equation}
\label{eq:L_h}
L \equiv (1 + D \epsilon) e^{2\gamma} + \frac{\varepsilon}{4 \gamma}.
\end{equation}
Then we can apply Lemma \ref{lem:robust_gd} with
\begin{equation}
\label{eq:T_eta_h}
T = \left\lceil \frac{16\gamma^2 L^2}{\varepsilon^2} \right\rceil, \qquad \eta = \frac{2\gamma}{L\sqrt{T}}.
\end{equation}
Using this directly in Lemma \ref{lem:robust_gd} gives us an estimate $\hat\theta$ such that
$\rc^*(\hat{\theta}) - \rc^*(\theta_u^*) \leq \varepsilon + 4\gamma e^{2\gamma} D \epsilon.$
Now, we can estimate the error parameter estimation error by bounding the curvature of the $\rc^*$ function at $\theta_u^*$
\begin{align}
 \delta \rc^*( \hat\theta) &\equiv  \rc^*(\hat \theta) - \rc^*(\theta^*) - \langle \nabla \rc^*(\theta^*), \hat \theta - \theta^* \rangle 
 \\&=  \rc^*(\hat\theta) - \rc^*(\theta^*)  \leq \varepsilon + 4 \gamma e^{2 \gamma} D \epsilon .
\end{align}

Now we can also lower bound the curvature of the function and derive the following result %(proof is in Appendix \ref{app:curv_bound})
\begin{lemma}\label{lem:strong_conv_h}
 The curvature term is lower bounded as $\delta \rc^*(\hat\theta) \geq \frac{e^{-\gamma}}{ 2 + 2\gamma} |\theta_u^* -\hat\theta|^2 $
\end{lemma}

\begin{proof}
Note, that the mechanism of this proof is basically the same as \ref{lem:strong_conv}. However, due to there being only one optimization variable, we do not have to resort to variance arguments at the end to prove a useful bound.

We work with the definition of the curvature function,
$$\delta \rc^*( \theta) \equiv  \rc(\theta) - \rc(\theta^*) -  \dfrac{d\rc(\theta^*)}{d \theta}(\theta - \theta^*) $$

Define local energy function $F(\usigma; \theta) = \sum_{k \neq u} \theta^*_{u,k}\sigma_u\sigma_k + \theta \sigma_u $. Notice that this function is linear in the $\theta$ variables.  So given an error parameter $\Delta \equiv   \theta- \theta^*$ on the parameters we can see that,

\begin{align}
   \rc^*(\theta) = \rc^*(\theta^* + \Delta)  =  \Emu \exp(-F(\usigma;\theta^* + \Delta))  =  \Emu \exp(-F(\usigma;\theta^*)) \  \exp(- \sigma_u \Delta).
\end{align}
The following is also true,
\begin{align}
     \dfrac{d \rc(\theta^*)}{d \theta} \Delta = -\Emu \exp(-F(\usigma;\utheta^*)) (\sigma_u \Delta).
\end{align}

The last two expressions together imply that given a function $g(x) = e^{-x} - 1 + x$, we can write the curvature term as,

\begin{align} \label{eq:mag_curv1}
    \delta \rc^*(\theta) = \Emu \exp(-F(\usigma, \theta^*))  g(\sigma_u \Delta)
\end{align}

First observe that $g(x) \geq  \frac{x^2}{2 + |x|}$ for all $x \in \mathbb{R}$ (See Lemma \ref{lem:notes2} for the proof)

Using this and the fact that $|F(\usigma, \theta^*)| \leq \gamma$ we get,

\begin{equation}\label{eq:mag_curv2}
    \delta \rc^*(\theta) \geq  \frac{e^{-\gamma}}{2 + 2\gamma }~\Emu (\sigma_u \Delta )^2 =   \frac{e^{-\gamma}}{2 + 2\gamma }~ ( \theta - \theta^*_u )^2 
\end{equation}

\end{proof}
Now using this we find that,
$\frac{e^{-\gamma}}{ 2 + 2\gamma} |\theta_u^* -\hat\theta|^2 \leq \varepsilon + 4 \gamma e^{2 \gamma} D \epsilon  $

Now define, $\epsilon_h \equiv \sqrt{\varepsilon 2(1+\gamma)} e^{\gamma/2}$. Using this we have,

$$ |\theta_u^* -\hat\theta| \ \leq \ \sqrt{\epsilon^2_h + 8 \gamma (1 +\gamma) e^{3\gamma} D \epsilon } \ \leq \ \epsilon_h + \sqrt{ 8 \gamma (1 +\gamma) e^{3\gamma} D \epsilon }  $$

\paragraph{Choice of $d$ and $\tau$.}
Recall that we chose the table accuracy $\tau$ and the Taylor truncation error $\varepsilon_{poly}$ so that
$\tau e^{2\gamma}=\varepsilon_{poly}=\varepsilon/(4\gamma)$.

\emph{Degree $d$.} By Lemma~\ref{lem:poly_exp} (applied with width parameter $2\gamma$), choosing $d+1>2\gamma$ ensures
\[
\varepsilon_{poly}\le \sup_{x\in[-2\gamma,2\gamma]}\left|e^{-x}-\sum_{k=0}^{d}\frac{(-x)^k}{k!}\right|
\le \exp\!\left(-(d+1)\Big(\log\!\frac{d+1}{2\gamma}-1\Big)\right).
\]
Therefore it suffices to choose $d$ so that
\[
\exp\!\left(-(d+1)\Big(\log\!\frac{d+1}{2\gamma}-1\Big)\right) \le \frac{\varepsilon}{4\gamma}.
\]
Equivalently, substituting $\varepsilon=\epsilon_h^2 e^{-\gamma}/(2(1+\gamma))$ into the Lambert-$W$ calculation (cf. Appendix~\ref{app:lW_calculation}), one can take
\begin{equation}
 d
= \left\lceil
\frac{\,\gamma + \log\!\dfrac{8\gamma(1+\gamma)}{\epsilon_h^2}\,}
{\,W\!\left(\frac{1}{2e} + \frac{1}{2\gamma e} \log\!\dfrac{8\gamma(1+\gamma)}{\epsilon_h^2}\right)}
\right\rceil - 1,
\end{equation}
where $W$ is the principal branch of the Lambert-$W$ function.

\emph{Table accuracy $\tau$.} By Lemma~\ref{lem:finite_sample_2}, on the event that the table is $\tau$-accurate, which has probability at least $1-\delta$, the monomial-based estimate of the truncated series in \eqref{eq:grad_taylor_h} has estimation error at most $\tau e^{2\gamma}=\varepsilon/(4\gamma)$ for every $u$. This gives
\[
 \tau = \frac{\varepsilon e^{-2\gamma}}{4\gamma} = \frac{\epsilon_h^2\, e^{-3\gamma}}{8\gamma(1+\gamma)} = \tau_h .
\]
Since the highest monomial degree appearing in \eqref{eq:exp3_h} is at most $d+1$, the table must have order $d+1-(d\bmod 2)$. Finally, Lemma~\ref{lem:emp_table} with $\frac{2}{\tau_h^2} = \frac{2^7\gamma^2(1+\gamma)^2e^{6\gamma}}{\epsilon_h^4}$ and $|\mathcal{K}_{d+1}| \le (ep)^{d+1}$ gives the stated sufficient sample size for empirical tables.

\end{proof}

\subsection{Proof of Lemma \ref{lem:finite_sample_2}}

The proof is the same as that of Lemma \ref{lem:finte_sample_1}. Work on the event $\mathcal{E}$ on which $\max_{K}|S[K]-m_K|\le\tau$ over all entries of the table, which has probability at least $1-\delta$.

Replacing the monomial expectations in \eqref{eq:exp3_h} by the corresponding table entries gives an approximation for the expectation value in the LHS, which we denote by $\widetilde{\EX}_S[\sigma^{k+1}_u (\sum_{j \neq u} \hat\theta_{uj} \sigma_j + \theta)^k ]$. Since the $\sigma$ variables square to $1$, each monomial $\sigma_u^{\,k+1} \prod_{j \ne u} \sigma_j^{\,m_j}$ equals $\prod_{i\in K}\sigma_i$ for a set $K$ with $|K|\le k+1 \le d+1$ and the parity of $k+1$, so the entry $S[K]$ exists in a table of order $d+1-(d\bmod 2)$.

From \eqref{eq:exp3_h}, we also see that,
\begin{align*}
 \sum_{\underline{m} \in \mathcal{M}_k}
|\hat{C}_{\theta, \underline{m}}|  = (\sum_{j \neq u} |\hat\theta_{uj}|+ |\theta|)^k 
\leq (2\gamma)^k .
\end{align*}
This is because we use the $\ell_1$ width as a constraint in the first round of optimization as well to find the second order parameters. On $\mathcal{E}$ this implies that,
\begin{align*}
\left|\EX[\sigma^{k+1}_u  (\sum_{j \neq u} \hat\theta_{uj} \sigma_j + \theta)^k ] - \widetilde{\EX}_S[\sigma^{k+1}_u  (\sum_{j \neq u} \hat\theta_{uj} \sigma_j + \theta)^k ] \right| \leq (2\gamma)^k \tau .
\end{align*}
Using this for each $k \leq  d$ in the truncated series in \eqref{eq:grad_taylor_h}, we get on $\mathcal{E}$ an estimation error of at most
\begin{align*}
\tau\sum_{k \leq d} \frac{(2\gamma)^k}{k!} \leq \tau e^{2\gamma}.
\end{align*}
The bound holds for every $u$ and every $|\theta|\le\gamma$ on the same event $\mathcal{E}$, so no further union bound is needed.

\section{Sample complexity of learning with structural assumptions} \label{app:last_theorem}
\begin{theorem}[Interaction screening with known $D$-regular structure]
\label{thm:known_struct}
Assume the interaction graph (edge set $E$) is known and $D$-regular, so $|N_u|=D$ for every $u\in[p]$. Let $\mathcal{K}_E \coloneqq \{K : K \subseteq \{u\}\cup N_u \text{ for some } u\in[p]\}$ be the set of local index sets, so that $|\mathcal{K}_E| \le p\,2^{D+1}$. Fix a target error $\epsilon\le 4\gamma$ and failure probability $\delta\in(0,1)$, and let
\[
\tau = \frac{\epsilon^2 e^{-4\gamma}}{2^{D+4}\gamma(1+\gamma)} .
\]
Suppose $S = \{S[K]\}_{K\in\mathcal{K}_E}$ satisfies $\mathbb{P}\big[\max_{K\in\mathcal{K}_E}|S[K]-m_K|\le\tau\big]\ge 1-\delta$; for example, the restriction of any $(D+1,\tau,\delta)$-approximate moment table to $\mathcal{K}_E$. For each $u$, let $\utheta_u = (\theta_u, (\theta_{u,v})_{v\in N_u}) \in \mathbb{R}^{D+1}$ collect the field and the couplings of $u$ to its neighbors, the couplings to non-neighbors being fixed to zero. Then projected gradient descent on the interaction screening loss over the feasible set $\{\utheta_u\in\mathbb{R}^{D+1}:\|\utheta_u\|_1\le\gamma\}$, using the entries of $S$ in place of the moments, returns with probability at least $1-\delta$ estimates satisfying, simultaneously for every $u\in[p]$,
\[
\max_{v\in N_u}|\hat\theta_{u,v}-\theta_{u,v}^*| \le \epsilon .
\]
In particular, it suffices to take
\[
T = \left\lceil
2^{8}(D+1)\,\gamma^2(1+\gamma)^2 e^{8\gamma}\,\epsilon^{-4}
\right\rceil,
\qquad
\eta = \frac{2\gamma}{L\sqrt{T}},
\qquad
L = 2\sqrt{D+1}\,e^{\gamma}.
\]
If $S$ consists of empirical moments built from a shared dataset or from separate datasets, then, by the argument of Lemma~\ref{lem:emp_table} applied to the $|\mathcal{K}_E|$ entries, it suffices that
\[
 n \ge
2^{2D+9}\,\gamma^2(1+\gamma)^2 e^{8\gamma}\,
\frac{(D+1)\log 2 + \log\!\left(\frac{2p}{\delta}\right)}{\epsilon^{4}} .
\]
\end{theorem}

\begin{proof}
There is no polynomial approximation error in this case unlike Theorem \ref{thm:main_thm}, so we only need to account for statistical error.
\begin{equation}
   \| \nabla \lc_u(\utheta) -  \widetilde\nabla \lc_{u}(\utheta) \|_\infty \leq \sum_K (|\nabla \widehat f_{u,K}(\utheta)| ) \max_{K \subseteq \{u\} \bigcup N_u} |m_K - S[K]| \coloneqq \varepsilon_{stat}
\end{equation}

Here the gradient is taken with respect to the $D+1$ coordinates of $\utheta_u$, and each coordinate of $\nabla_{\utheta} f_u(\usigma;\utheta)$ is a function on $\{-1,1\}^{D+1}$. Using the standard Fourier (Walsh--Hadamard) expansion on the boolean hypercube (see, e.g., \cite[Ch.~1]{o2014analysis}), we may write
\[
\nabla_{\utheta} f_u(\usigma;\utheta)
= \sum_{K\subseteq\{u\}\cup N_u} \nabla \widehat f_{u,K}(\utheta)\,\usigma_K,
\qquad
\nabla \widehat f_{u,K}(\utheta)
= 2^{-(D+1)}\sum_{\mathbf{s}\in\{-1,1\}^{D+1}} \nabla_{\utheta} f_u(\mathbf{s};\utheta)\,s_K,
\]
so that $\nabla \lc_u(\utheta) = \sum_{K} \nabla \widehat f_{u,K}(\utheta)\, m_K$ and $\widetilde\nabla \lc_u(\utheta) = \sum_{K} \nabla \widehat f_{u,K}(\utheta)\, S[K]$, which gives the display above. Each coordinate of $\nabla_{\utheta} f_u(\mathbf{s};\utheta)$ is of the form $-s_u s_v f_u(\mathbf{s};\utheta)$ or $-s_u f_u(\mathbf{s};\utheta)$, and $|f_u(\mathbf{s};\utheta)| \le e^{\|\utheta\|_1} \le e^{\gamma}$ on the feasible set. Since $s_K\in\{-1,1\}$, we have the uniform bound
\[
\bigl\|\nabla \widehat f_{u,K}(\utheta)\bigr\|_\infty
\le \max_{\mathbf{s}\in\{-1,1\}^{D+1}}\bigl\|\nabla_{\utheta} f_u(\mathbf{s};\utheta)\bigr\|_\infty
\le e^{\gamma},
\]
and therefore
\[
\sum_{K\subseteq\{u\}\cup N_u}\bigl\|\nabla \widehat f_{u,K}(\utheta)\bigr\|_\infty
\le 2^{D+1} e^{\gamma}.
\]

Next we bound the moment errors. On the event $\mathcal{E}$ that $\max_{K\in\mathcal{K}_E}|S[K]-m_K|\le\tau$, which has probability at least $1-\delta$, we have for every $u$,
\[
\varepsilon_{stat} \le 2^{D+1} e^{\gamma}\, \tau.
\]

\paragraph{Robust gradient descent.}
Having established a uniform bound
\(
\sup_{\|\utheta\|_1\le\gamma}\|\nabla \lc_u(\utheta)-\widetilde\nabla\lc_u(\utheta)\|_{\infty}\le \varepsilon_{stat}
\), we can invoke Lemma~\ref{lem:robust_gd} (robustness of projected gradient descent) exactly as in the proof of Theorem~\ref{thm:main_thm}.
In the notation of Lemma~\ref{lem:robust_gd}, we may take
\[
\varepsilon \coloneqq 4\gamma\,\varepsilon_{stat},
\qquad
L \coloneqq 2\sqrt{D+1}\,e^{\gamma},
\]
where $L$ bounds the $\ell_2$ norm of the approximate gradient, as required by Lemma~\ref{lem:robust_gd}. Indeed, $\|\widetilde\nabla\lc_u(\utheta)\|_{\infty}\le \|\nabla\lc_u(\utheta)\|_{\infty} + \varepsilon_{stat} \le e^{\gamma} + \varepsilon_{stat} \le 2e^{\gamma}$, where the last step uses the choice of $\varepsilon_{stat}$ below and $\epsilon\le 4\gamma$, which give $\varepsilon_{stat}\le 2\gamma e^{-3\gamma}/(1+\gamma) \le e^{\gamma}$; since the local parameter vector has dimension $D+1$, $\|\widetilde\nabla\lc_u(\utheta)\|_{2}\le \sqrt{D+1}\,\|\widetilde\nabla\lc_u(\utheta)\|_{\infty}\le L$.
Lemma~\ref{lem:robust_gd} then yields the iteration complexity and step size
\[
T = \left\lceil \frac{16\gamma^2 L^2}{\varepsilon^2} \right\rceil
= \left\lceil \frac{L^2}{\varepsilon_{stat}^2} \right\rceil,
\qquad
\eta = \frac{2\gamma}{L\sqrt{T}},
\]
which ensure that the averaged iterate satisfies $\lc_u(\bar\utheta^T)-\lc_u(\utheta_u^*)\le \varepsilon$.

\paragraph{From loss error to parameter error (curvature).}
Finally, we convert the loss optimality gap into a coordinate-wise parameter error bound using the same curvature argument as in the proof of Theorem~\ref{thm:main_thm}.
Let $\uDelta_u\coloneqq \bar\utheta^T-\utheta_u^*$. Define the curvature functional
\[
\delta \lc_u^*(\bar\utheta^T)
\coloneqq \lc_u(\bar\utheta^T)-\lc_u(\utheta_u^*)-\langle\nabla\lc_u(\utheta_u^*),\uDelta_u\rangle.
\]
Since $\utheta_u^*$ is a minimizer of the convex loss over the feasible set, we have $\delta \lc_u^*(\bar\utheta^T)\le \lc_u(\bar\utheta^T)-\lc_u(\utheta_u^*)\le \varepsilon$.
Moreover, Lemma~\ref{lem:strong_conv} (whose proof applies unchanged when the couplings to non-neighbors are fixed to zero) yields the lower bound
\[
\delta \lc_u^*(\bar\utheta^T)
\ge \frac{e^{-3\gamma}}{2+2\gamma}\,\max_{v\in N_u}|\Delta_{u,v}|^2 .
\]
As in Theorem~\ref{thm:main_thm}, this bound does not control the error on the field $\theta_u$, so the guarantee is stated for the couplings only.
Combining the two displays gives
\[
\max_{v\in N_u}|\bar\theta^T_{u,v}-\theta_{u,v}^*|
\le e^{3\gamma/2}\sqrt{2(1+\gamma)\,\varepsilon}
= e^{3\gamma/2}\sqrt{8\gamma(1+\gamma)\,\varepsilon_{stat}}.
\]

To achieve a target accuracy $\max_{v\in N_u}|\bar\theta^T_{u,v}-\theta_{u,v}^*|\le \epsilon$ with overall success probability at least $1-\delta$, it suffices to set
\[
\varepsilon_{stat}
\le \frac{\epsilon^{2} e^{-3\gamma}}{8\gamma(1+\gamma)}
\]
which, by the bound on $\varepsilon_{stat}$ above, holds on $\mathcal{E}$ if $\tau = \frac{\epsilon^2 e^{-4\gamma}}{2^{D+4}\gamma(1+\gamma)}$.
Since the table $S$ is fixed before running the optimization and reused across all $T$ iterations and all nodes $u$, the single event $\mathcal{E}$ implies the gradient approximation guarantee holds uniformly for every iterate and every node, and no additional union bound over the $T$ steps or the $p$ nodes is needed.

Finally, for empirical tables, each entry is an average of $n$ i.i.d.\ random variables in $\{-1,1\}$, so Hoeffding's inequality and a union bound over the $|\mathcal{K}_E| \le p\,2^{D+1}$ entries give $\mathbb{P}[\mathcal{E}^c] \le 2p\,2^{D+1}e^{-n\tau^2/2}$, which is at most $\delta$ when $n \ge \frac{2}{\tau^2}\big((D+1)\log 2 + \log\frac{2p}{\delta}\big)$. Substituting $\frac{2}{\tau^2} = \frac{2^{2D+9}\gamma^2(1+\gamma)^2e^{8\gamma}}{\epsilon^4}$ gives the stated sample size.

\end{proof}

\section{Numerical experiments}\label{app:experiments}

Our results are worst-case upper bounds. In this section we run Algorithm~\ref{alg:ais-moments} on a small model to see how the moment order needed in practice depends on the $\ell_1$ width, and how it compares with the order prescribed by Theorem~\ref{thm:main_thm}. We also check whether restricting the learner to low-order empirical moments costs extra samples compared with Interaction Screening on full samples.

\paragraph{Setup.} We use the long-range ferromagnetic chain on $p=12$ spins with couplings $\theta^*_{uv} = J/(u-v)^2$ for all $u \neq v$, with $J>0$, and no magnetic fields. Every spin interacts with every other spin, and the $\ell_1$ width $\gamma$ is proportional to $J$, so it can be tuned by changing $J$. With $p=12$, the distribution, and hence every moment, can be computed exactly by enumerating the $2^{12}$ configurations.
We run projected gradient descent with the approximate gradients of Algorithm~\ref{alg:ais-moments}, the $\ell_1$ radius set to the true width $\gamma$, step size $\eta = 0.2$, and $3000$ iterations ($2000$ in Experiment~2). We report the last iterate rather than the average; with these settings the iterates have converged, and running four times as many iterations does not change the reported errors. Unlike in our theorems, we use the constant step size and number of iterations above rather than the values of $\eta$ and $T$ in Theorem~\ref{thm:main_thm}, which are tuned for the worst case. The error is $\max_{u \neq v}|\hat\theta_{uv} - \theta^*_{uv}|$, where row $u$ of $\hat\theta$ is the estimate obtained from the loss of spin $u$.

\paragraph{Computing the approximate gradient.} When every entry of the table is a moment of a single distribution $q$ (the exact table, or an empirical table built from a shared dataset, in which case $q$ is the empirical distribution of the samples), steps (ii) and (iii) of Algorithm~\ref{alg:ais-moments} can be carried out in closed form. Expanding into monomials and substituting $S[K] = \EX_q[\prod_{i\in K}\sigma_i]$ gives
\[
\tilde{\nabla}_{uv}\lc_{u,S}(\utheta_u) = -\,\EX_{q}\Big[\sigma_u\sigma_v \sum_{k\le d}\frac{(-E_u(\usigma,\utheta_u))^k}{k!}\Big],
\]
that is, the Interaction Screening gradient with $e^{-E_u}$ replaced by its degree-$d$ Taylor polynomial and the expectation taken under $q$. We use this form in our experiments, and we checked numerically that it agrees with the monomial expansion to machine precision. Note that the learner still only uses moments of order at most $D = d+2-(d\bmod 2)$. Taking $d=\infty$ in this formula gives Interaction Screening with full samples, which we use as a baseline.

\paragraph{Experiment 1: moment order versus width.} Using the exact moment table, we find the smallest degree $d$ for which the error is at most $\epsilon = 0.02$, for $\gamma \in \{0.25, 0.5, \ldots, 2.5\}$. The results are shown in Figure~\ref{fig:degree_vs_width}, together with the degree prescribed by Theorem~\ref{thm:main_thm} for the same $\epsilon$ and $\gamma$.

\begin{figure}[h]
\centering
\includegraphics[width=0.7\linewidth]{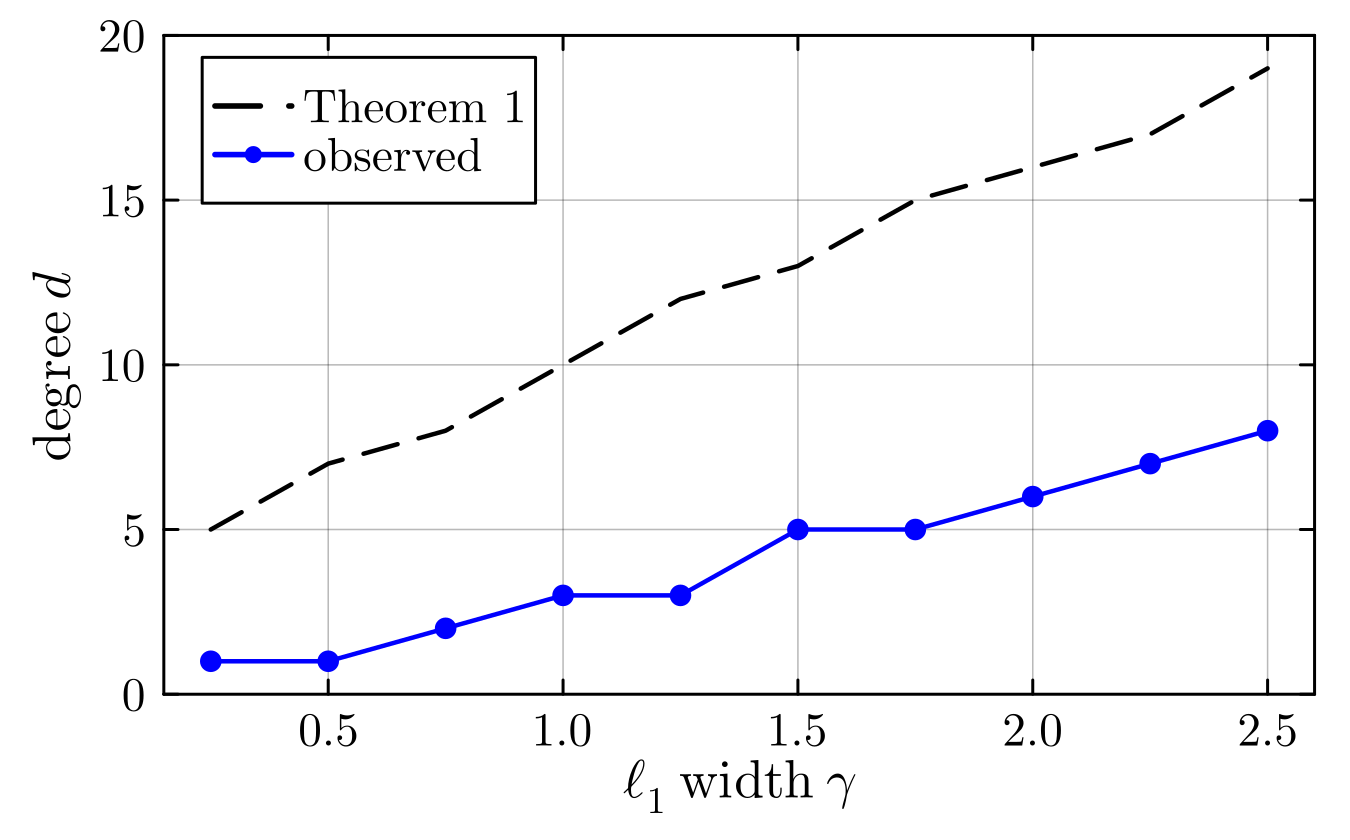}
\caption{Smallest degree $d$ of the Taylor approximation for which Algorithm~\ref{alg:ais-moments}, run on the exact moment table of order $d+2-(d\bmod 2)$, learns all couplings of the long-range chain to error $0.02$ ($p=12$). The dashed line is the degree prescribed by Theorem~\ref{thm:main_thm}.}
\label{fig:degree_vs_width}
\end{figure}

The required degree grows roughly linearly with $\gamma$, as predicted by Theorem~\ref{thm:main_thm}. The prescribed degree is always larger than the observed one, by a factor between about $2.4$ and $7$, as expected from a worst-case bound.

\paragraph{Experiment 2: number of samples.} We now replace the exact table by an empirical table built from a shared dataset of $n$ i.i.d.\ samples, for the long-range chain with $\gamma = 1$. For this model, $d=3$ (table order $4$) is the smallest degree that reaches error $0.02$ with the exact table. Figure~\ref{fig:error_vs_samples} shows the error for $d \in \{1,3,6\}$ and for Interaction Screening with full samples ($d=\infty$), with the same dataset used for every $d$ in each trial. Points are medians over $10$ independent datasets, and error bars show the interquartile range.

\begin{figure}[h]
\centering
\includegraphics[width=0.7\linewidth]{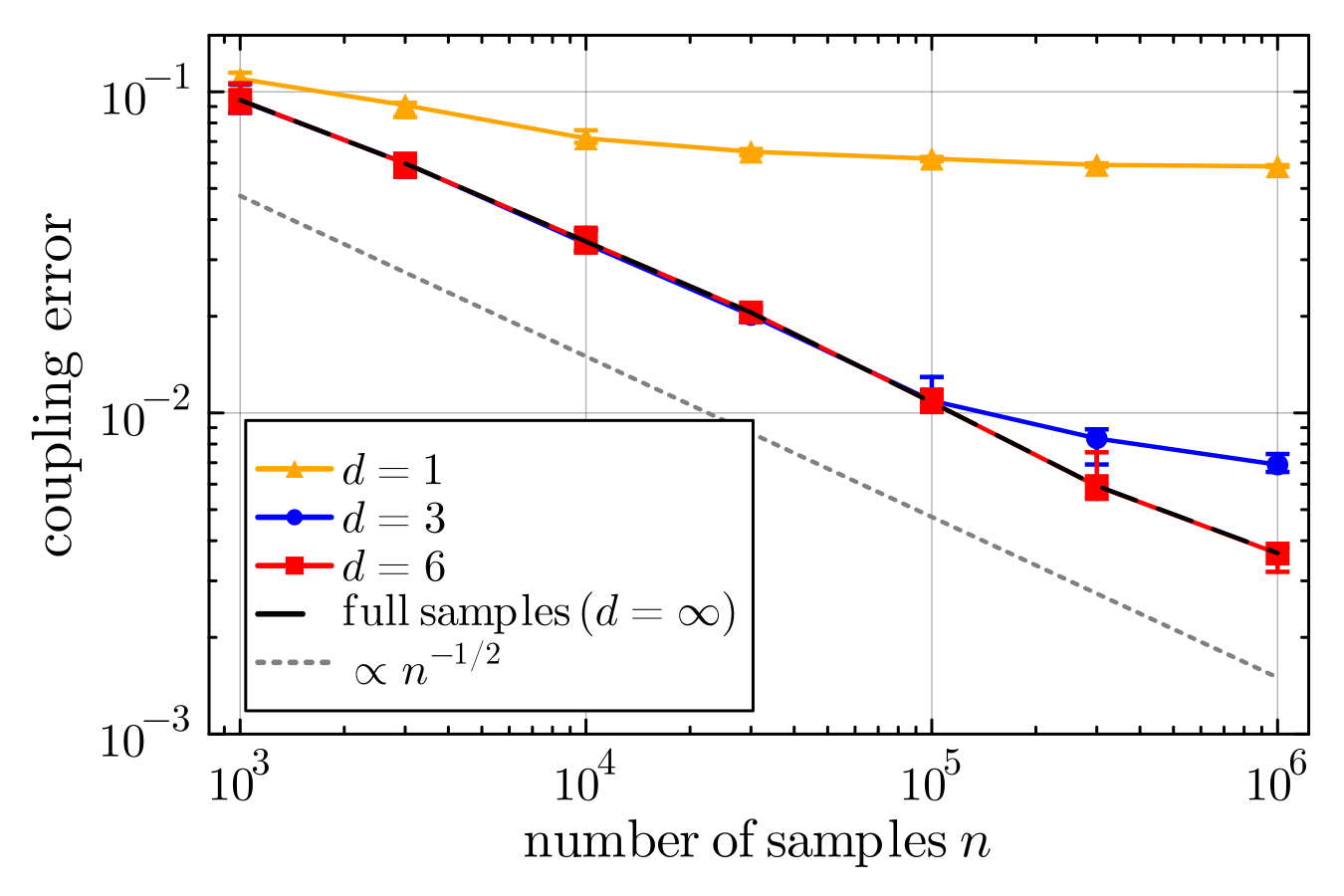}
\caption{Coupling error of Algorithm~\ref{alg:ais-moments} on empirical moment tables built from a shared dataset of $n$ samples, for the long-range chain with $\gamma=1$ ($p=12$), compared with Interaction Screening on the full samples. The curves for $d=6$ and for full samples coincide. Medians over $10$ datasets; error bars show the interquartile range.}
\label{fig:error_vs_samples}
\end{figure}

For $d=6$ the error is indistinguishable from that of Interaction Screening with full samples over the whole range of $n$, and both decrease at the rate $n^{-1/2}$. For $d=3$ the same holds up to $n \approx 10^5$. Beyond that point its error levels off, approaching $0.005$, which is its error on the exact table. With $d=1$ (table order $2$), the error levels off much earlier, at about $0.06$, which is its error on the exact table. This agrees with the main message of our results: once the moment order is large enough for the polynomial approximation error to be small, using only low-order moments costs no additional samples. The sample sizes prescribed by Corollary~\ref{cor:sample_complexity} are, as expected, very conservative. For $\gamma=1$, $\epsilon = 0.02$ and $\delta = 0.05$ they exceed $10^{15}$, whereas $n \approx 10^5$ samples suffice here.

%\section{Score matching heuristic to learn Ising models}

%Here we will describe a score based heuristic to learn Ising models. The idea is to pretend that the statistics of the Ising model are actually coming from a continuous variable model of the following form with a large, but known $\Lambda$ to approximately pin the variables to $\{ +1,-1 
%\}^p$
%$$\mu(\underline{x}) \propto  e^{\sum_{u \neq v} \theta^*_{u,v} x_u x_v  +  \sum_{u}\theta^*_u x_u - \Lambda \sum_u (x^2_u-1)^2}.$$

%For a parametric energy function, $E(\underline x, \Theta)$  over variables $\underline x$ parameterized by some set of parameters $\Theta$   the score matching loss can be written as,

%\begin{equation}
%   \lc_{score} (\Theta) = \sum_{u = 1} ^ p \EXp{\underline x \sim\mu}  \left(\dfrac{\partial^2 E(\underline x, \Theta) }{\partial x^2_u} + \frac12  \left(\dfrac{\partial E(\underline x, \Theta) }{\partial x_u}\right)^2 \right )
%\end{equation}

%\newpage
%\input{checklist.tex}

\end{document}